 \documentclass[preprint,3p,twocolumn]{elsarticle}

\usepackage{amssymb}
\usepackage{lipsum}
\usepackage{amsmath} 
\usepackage{hyperref}
\usepackage{amsthm}
\usepackage{stfloats}

\newtheorem{theorem}{Theorem}
\date{}

\journal{arXiv}

\begin{document}

\begin{frontmatter}



\title{Synergistic eigenanalysis of covariance and Hessian matrices for enhanced binary classification on health datasets}


\affiliation[first]{organization={Sano - Centre for Computational Personalized Medicine},
            city={Krakow},
            country={Poland}}
\affiliation[third]{organization={Jagiellonian University},
            city={Krakow},
            country={Poland}}
\affiliation[second]{organization={School of Computing, Telkom University},
            city={Bandung},
            country={Indonesia}}

\author[first,second]{Agus Hartoyo \corref{cor1}}
\author[first,third]{Jan Argasiński}
\author[third]{Aleksandra Trenk}
\author[third]{Kinga Przybylska}
\author[first,third]{Anna Blasiak}
\author[first]{Alessandro Crimi}

\cortext[cor1]{Corresponding author. Sano - Centre for Computational Personalized Medicine, Czarnowiejska 36 building C5, 30-054 Kraków. E-mail address: a.hartoyo@sanoscience.org}

\begin{abstract}
Covariance and Hessian matrices have been analyzed separately in the literature for classification problems. However, integrating these matrices has the potential to enhance their combined power in improving classification performance. We present a novel approach that combines the eigenanalysis of a covariance matrix evaluated on a training set with a Hessian matrix evaluated on a deep learning model to achieve optimal class separability in binary classification tasks. Our approach is substantiated by formal proofs that establish its capability to maximize between-class mean distance (the concept of \textit{separation}) and minimize within-class variances (the concept of \textit{compactness}), which together define the two linear discriminant analysis (LDA) criteria, particularly under ideal data conditions such as isotropy around class means and dominant leading eigenvalues. By projecting data into the combined space of the most relevant eigendirections from both matrices, we achieve optimal class separability as per these LDA criteria. Empirical validation across neural and health datasets consistently supports our theoretical framework and demonstrates that our method outperforms established methods. Our method stands out by addressing both separation and compactness criteria, unlike PCA and the Hessian method, which predominantly emphasize one criterion each. This comprehensive approach captures intricate patterns and relationships, enhancing classification performance. Furthermore, through the utilization of both LDA criteria, our method outperforms LDA itself by leveraging higher-dimensional feature spaces, in accordance with Cover's theorem, which favors linear separability in higher dimensions.  Additionally, our approach sheds light on complex DNN decision-making, rendering them comprehensible within a 2D space.
\end{abstract}



\begin{keyword}
covariance matrix \sep Hessian matrix \sep eigenanalysis \sep binary classification \sep class separability



\end{keyword}

\end{frontmatter}




\section{Introduction}

Binary classification is a fundamental task in machine learning, where the goal is to assign data points to one of two classes. The accuracy and effectiveness of binary classifiers depend on their ability to separate the two classes accurately. However, achieving optimal class separability can be challenging, especially when dealing with complex and high-dimensional data.

Traditional approaches often rely on analyzing either the covariance matrix \citep{Nagai2020TheCM,Minh2017CovariancesIC, Serra2014CovarianceOC, Lenc2016LearningCF, Hoff2011ACR, Kuo2002ACE, Lam2019HighdimensionalCM} or the Hessian matrix \citep{Dawid2021HessianbasedTF, Fu2020HesGCNHG, Yao2019PyHessianNN, Krishnasamy2016HessianSE, Wiesler2013InvestigationsOH, Byrd2011OnTU, Martens2010DeepLV} separately to optimize machine learning models. In the field of Evolution Strategies (ESs), a recent work \citep{shir2020covariance} explored the relationship between the covariance matrix and the landscape Hessian, highlighting the statistical learning capabilities of ESs using isotropic Gaussian mutations and rank-based selection. While this study provides valuable insights into ESs' learning behavior, it does not investigate the practical integration of the two matrices. 

When analyzing the covariance matrix, the focus is on capturing the inherent patterns of variability within the data, allowing for a compact representation that highlights relationships between different dimensions. On the other hand, the analysis of the Hessian matrix aims to find the direction in which the classes are best separated, by maximizing the curvature along the discriminative directions. However, these separate analyses fail to fully leverage the synergistic effects that can arise from integrating the information contained in both matrices.

To tackle this challenge, we present a novel approach that combines the eigenanalysis of the covariance matrix evaluated on a training set with the Hessian matrix evaluated on a deep learning model. By integrating these matrices, our method aims to optimize class separability by simultaneously maximizing between-class mean distance and minimizing within-class variances, which are the fundamental criteria of linear discriminant analysis (LDA) \citep{fisher1936use, xanthopoulos2013linear}. By utilizing these criteria, we leverage the foundational principles that have been established and proven over decades of LDA's application in various fields like medicine \citep{sharma2012between, sharma2008cancer, moghaddam2006generalized, dudoit2002comparison, chan1995computer}, agriculture \citep{tharwat2017one, gaber2015plant, rezzi2005classification, heberger2003principal, chen1998rapid}, and biometrics \citep{paliwal2012improved, yuan2007ear, park2005fingerprint, wang2004random, yu2001direct, chen2000new, haeb1992linear}. LDA has been demonstrating the reliability and usefulness of the two criteria as indicators of discriminative power. Thus, by building upon this well-established foundation, our approach will inherit the strengths and reliability that have been demonstrated by LDA. This integrated approach holds the promise of enhancing classification performance beyond conventional methods that treat the two matrices separately.

Clustering algorithms \citep{lloyd1982least, forgy1965cluster, kaufman1990partitioning, kaufman2009finding, johnson1967hierarchical, zhang1996birch, andrews1973numerical, hidalgo2022inferring, ester1996density, campello2013density, ankerst1999optics, agrawal1998automatic} also adhere to principles similar to those of LDA criteria. Just as our proposed method leverages the LDA criteria for optimizing class separability by maximizing between-class mean distance and minimizing within-class variances, clustering techniques aim to enhance cluster separability by maximizing both inter-cluster dissimilarity (the concept of \textit{separation}) and intra-cluster similarity (the concept of \textit{compactness}) \citep{mining2006data, mining2006introduction, zhao2002evaluation, liu2010understanding}. This alignment highlights the effectiveness and utility of LDA criteria in various data analysis techniques, demonstrating that these principles are broadly applicable and beneficial for achieving distinct and compact data groupings.

\section{Methodology}

\subsection{Approach: combining eigenanalysis of covariance and Hessian matrices}

In this subsection, we describe our novel approach for combining the eigenanalysis of the covariance matrix and the Hessian matrix to achieve optimal class separability in binary classification tasks.

\begin{enumerate}

\item Covariance matrix eigenanalysis: A covariance matrix $S(\boldsymbol{\theta})$ is a $D \times D$ matrix, where $D$ represents the dimensionality of the predictor attributes. Each element of the covariance matrix reflects the covariance between the two predictor attributes $\theta_1$ and $\theta_2$:

\begin{equation} \label{eq:cov}
S(\theta_1, \theta_2) = \frac{1}{n-1} \sum_{i=1}^{n} (x_{1i} - \bar{x}_1)(x_{2i} - \bar{x}_2)
\end{equation}

where \(x_{1i}\) and \(x_{2i}\) are the corresponding observations of these predictor attributes for the \(i\)-th data instance, \(\bar{x}_1\) and \(\bar{x}_2\) are the sample means of the predictor attributes $\theta_1$ and $\theta_2$, respectively, and \(n\) is the number of data instances or observations.

Performing eigenanalysis on the covariance matrix $S(\boldsymbol{\theta})$, we obtain eigenvalues $\lambda_i$ and corresponding eigenvectors $\mathbf{v}_i$. The leading eigenvector $\mathbf{v}_1$ associated with the largest eigenvalue $\lambda_1$, as captured in the eigen-equation $S(\boldsymbol{\theta}) \cdot \mathbf{v}_1 = \lambda_1 \cdot \mathbf{v}_1 $, represents the principal direction with the highest variance.

\item Hessian matrix eigenanalysis:
Next, we compute a Hessian matrix evaluated on a deep learning model trained on the same training set. We employ a deep neural network (DNN) architecture consisting of four fully connected layers, each followed by a Rectified Linear Unit (ReLU) activation. The final layer employs a sigmoid activation function to yield a probability value within the range of 0 to 1. During training, we utilize the binary cross-entropy loss function. The binary cross-entropy loss is given by

\begin{equation} 
\text{BCELoss} = - \sum_{i=1}^n \log p_\theta(c_i \mid x_i), 
\end{equation}

where $c_i$ is the actual class for data instance $x_i$. The probability $p_\theta(c_i \mid x_i)$ is the predicted probability of instance $x_i$ belonging to class $c_i$, as estimated by the DNN model parameterized by $\theta$.

The Hessian of the loss function is then:

\begin{equation} \label{eq:hess}
\begin{aligned}
H_\theta &= \nabla_\theta^2 \text{BCELoss} \\
&= \nabla_\theta^2 \left[ - \sum_{i=1}^n \log p_\theta(c_i \mid x_i) \right] \\
&= - \sum_{i=1}^n \left[\nabla_\theta^2 \log p_\theta(c_i \mid x_i) \right].
\end{aligned}
\end{equation}

Performing eigenanalysis on the Hessian matrix $H_{\boldsymbol{\theta}}$, we obtain eigenvalues $\lambda_i'$ and corresponding eigenvectors $\mathbf{v}_i'$. The leading eigenvector $\mathbf{v}_1'$ associated with the largest eigenvalue $\lambda_1'$, as captured in the eigen-equation $H_{\boldsymbol{\theta}} \cdot \mathbf{v}_1' = \lambda_1' \cdot \mathbf{v}_1'$, represents the direction corresponding to the sharpest curvature.

\item Integration of matrices and projection of data:

To combine the power of covariance and Hessian matrices, we project the data into the combined space of the most relevant eigendirections. Let $\mathbf{U}$ be the matrix containing the leading eigenvectors from both matrices:

\begin{equation} \label{eq:a}
\mathbf{U} = [\mathbf{v}_1, \mathbf{v}_1']
\end{equation}

The 2D projection of the data, denoted as $\mathbf{X}_{\text{proj}}$, is obtained by:

\begin{equation} \label{eq:b}
\mathbf{X}_{\text{proj}} = \mathbf{X} \cdot \mathbf{U}
\end{equation}

Here, $\mathbf{X}$ is the original data matrix, and $\mathbf{X}_{\text{proj}}$ represents the final output of the proposed method—a 2D projection capturing both statistical spread and discriminative regions of the data.

Unlike LDA, which aims to optimize both criteria simultaneously along a single direction for binary classification, constrained by the limitation that the number of linear discriminants is at most \(c-1\) where \(c\) is the number of class labels \cite{ye2004two}, our approach provides more flexibility and control. By working on two separate directions, we specifically focus on minimizing the within-class variances in one direction while maximizing the between-class mean distance in the other direction.

\end{enumerate}

\subsection{Formal foundation: maximizing the squared between-class mean distance and minimizing the within-class variance}

In this subsection, we establish two theorems along with their respective proof sketches that form the theoretical basis for our approach. Full proofs for both theorems are available in \ref{full-proof}.

\begin{theorem}[Maximizing covariance for maximizing squared between-class mean distance]\label{thm:Theorem1}
Consider two sets of 1D data points representing two classes, denoted as $C_1$ and $C_2$, each consisting of $n$ samples. The data in $C_1$ and $C_2$ are centered around their respective means, $\mu_1$ and $\mu_2$. Here, \( \mu \) denotes the overall mean of the combined data from $C_1$ and $C_2$. Furthermore, the variances for \(C_1\) and \(C_2\) are denoted as \(\sigma_1^2\) and \(\sigma_2^2\), respectively, and the combined data has an overall variance of \(\sigma^2\). The between-class mean distance, denoted as $d$, represents the separation between the means of $C_1$ and $C_2$. We establish the following relationship:

\begin{equation} \label{eq:1}
\sigma^2 = \frac{d^2}{4(1 - \lambda)}, 
\end{equation}
where $\lambda = \frac{1}{2} \left( \frac{\sigma_1^2}{\sigma^2} + \frac{\sigma_2^2}{\sigma^2} \right)$ means a constant between 0 and 1 that reflects the distribution of the original data when \(C_1\) and \(C_2\) are considered as projected representations, which implies that maximizing the variance ($\sigma^2$) in this context will also maximize the squared between-class mean distance ($d^2$).

\end{theorem}

\paragraph{Proof sketch} 

\begin{enumerate}
  \item Consider two classes \(C_1\) and \(C_2\) with identical underlying distributions.
  \item Simplify the expression for the combined data variance $\sigma^2$ to $\frac{1}{2}\left(\sigma_1^2 + \sigma_2^2  \right) + \frac{1}{4}d^2$.
  \item Apply \hyperref[VRPT]{the variance ratio preservation property for projection onto a vector} to relate $\sigma_1^2$, $\sigma_2^2$, and $\sigma^2$ through $\lambda$.
  \item Simplify the expression to obtain $\sigma^2 = \frac{d^2}{4(1 - \lambda)}$.
  \item Conclude that maximizing $\sigma^2$ maximizes $d^2$, fulfilling the theorem's objective.
\end{enumerate}
\hfill \qedsymbol

In the context of multidimensional data, extending Theorem \ref{thm:Theorem1} involves considering data points as $D$-dimensional vectors, where the dataset consists of two subsets, \(C_{1} = \{\vec{x}_{i} \in \mathbf{R}^{D} \mid i=1,2,\ldots,n\}\) with covariance \(S_1\), and \(C_{2} = \{\vec{y}_{j} \in \mathbf{R}^{D} \mid j=1,2,\ldots,n\}\) with covariance \(S_2\). The theorem, however, is supposed to hold for the projection of the data onto an arbitrary 1D subspace of the full D-dimensional data space, and to imply that maximizing $\sigma^2$ by choosing the direction of this subspace appropriately will also be the direction that maximizes the difference between the projected sample means. A general case in which it holds is when the individual distributions are isotropic about their means, meaning the covariances $S_1$ and $S_2$ are proportional to the identity matrix in $\mathrm{D}$ dimensions as mathematically shown in \ref{extension}.

\begin{theorem}[Maximizing Hessian for minimizing within-class variance]\label{thm:Theorem2}
Let $\theta$ be a parameter of the model, and $\mathrm{H}_\theta$ denote the Hessian of the binary cross-entropy loss function with respect to $\theta$. We define the within-class variance as the variance of a posterior distribution $p_\theta(\theta \mid c_i)$, which represents the distribution of the parameter $\theta$ given a class $c_i$. We denote the variance of this posterior distribution as $\sigma_{post}^2$. We establish the following relationship:

\begin{equation} \label{eq:2}
\mathrm{H}_\theta = \frac{1}{\sigma_{\text{post}}^2},
\end{equation}
which implies that maximizing the Hessian ($\mathrm{H}_\theta$) will minimize the within-class variances ($\sigma_{\text{post}}^2$).

\end{theorem}

\paragraph{Proof sketch}
\begin{enumerate}
  \item Consider the parameter $\theta$ and the Hessian $\mathrm{H}_\theta$ of the binary cross-entropy loss function. Define the within-class variance $\sigma_{\text{post}}^2$ as the variance of the posterior distribution $p_\theta(\theta \mid c_i)$.
  \item Approximate the Hessian $\mathrm{H}_\theta$ as the Fisher information using the expectation of the squared gradient of the log-likelihood \citep{barshan2020relatif}.
  \item Assume a known normal likelihood distribution for $p_\theta(c_i \mid \theta)$ with mean $\mu$ and standard deviation $\sigma$. Compute the Fisher information as $\frac{1}{\sigma^2}$.
  \item Considering a uniform prior distribution $p(\theta)$ within the plausible range of $\theta$ and the evidence $p(c_i)$ as a known constant, apply Bayes' formula to derive that $\sigma_{\text{post}}^2 = \sigma^2$.
  \item Derive $\mathrm{H}_\theta = \frac{1}{\sigma_{\text{post}}^2}$.
  \item Conclude that maximizing $\mathrm{H}_\theta$ minimizes $\sigma_{\text{post}}^2$, achieving the theorem's objective.
\end{enumerate}
\hfill \qedsymbol

Theorem \ref{thm:Theorem1} and Theorem \ref{thm:Theorem2} respectively suggest that maximizing the variance of projected data and the Hessian effectively maximize squared between-class mean distance and minimize within-class variances, contingent upon the ideal data conditions discussed in Section \ref{ideal_condition}. These theorems provide the theoretical foundation for the eigenanalysis of the covariance and Hessian matrices, crucial steps in our proposed method for improving class separability based on the two LDA criteria.

\section{Ideal dataset conditions for the optimal performance of the proposed method}
\label{ideal_condition}

\subsection{General guidelines}

To ensure the optimal performance of the proposed method, it is crucial to consider certain ideal conditions for the dataset. These conditions represent ideal scenarios that can help achieve optimal performance and support the theoretical foundations of the method. The key conditions are as follows:

\begin{enumerate}
    \item Isotropic distributions around class means:
    
    The theoretical underpinnings of the proposed method, particularly in the application of Theorem \ref{thm:Theorem1} to multidimensional data, benefit from the dataset subsets corresponding to each class being approximately isotropic around their respective means, as mathematically shown in \ref{extension}. While perfect isotropy—where the covariance matrices of the subsets are proportional to the identity matrix in \(D\) dimensions—is rare in practice, achieving near-isotropic conditions helps ensure the validity of the theorem to higher dimensions and the method's ability to leverage the relationship between covariance and class separability.

    In more practical terms, datasets tend to support isotropism under the following conditions:
    \begin{enumerate}
        \item \textit{Normalized data:} Normalizing the data, typically through techniques such as z-score normalization, ensures that each attribute has a mean of zero and a standard deviation of one. This process equalizes the variance in all directions across both classes, which, as discussed in Subsection \ref{sec:z-score} and formalized in Theorem \ref{thm:Theorem3}, helps approximate the isotropic condition within each class.

        \item \textit{Low correlation between attributes:} Datasets with low correlation between attributes produce covariance matrices with minimized off-diagonal elements. This condition approximates isotropy, which is characterized by covariance matrices being proportional to the identity matrix. Low correlation is also a desirable property of a well-structured dataset, as high correlation (multicollinearity) can undermine the statistical significance of individual variables and reduce the interpretability and robustness of the model \cite{kyriazos2023dealing, kalnins2018multicollinearity, dertli2024effects}. 
    \end{enumerate}

    \item Dominance of the first eigenvalues in covariance and Hessian matrices:
    
    The effectiveness of the proposed method is heavily reliant on the information provided by the leading eigenvectors derived from the covariance and Hessian matrices. For these matrices to significantly contribute to the method's success, their first eigenvalues need to be sufficiently dominant. A dominant first eigenvalue indicates that the corresponding eigenvector captures the most significant variance or curvature in the data, making it a crucial direction for class separation.  

    The presence of a dominant first eigenvalue can be assessed from the eigenspectra of the two matrices, which display the eigenvalues on a log scale. A dominant first eigenvalue is characterized by the eigenvalues being spread over many orders of magnitude, often with approximately uniform spacing between them. When a dominant first eigenvalue is present in the covariance matrix, it suggests a dominant underlying structure or trend in the data. When this occurs in the Hessian matrix, it characterizes a "sloppy" model \cite{waterfall2006sloppy, transtrum2010nonlinear, transtrum2011geometry, transtrum2015perspective, machta2013parameter, raman2017delineating, hartoyo2019parameter}, which has a few stiff directions (significant eigenvalues) and many sloppy directions (small eigenvalues). 
\end{enumerate}

\subsection{Z-score normalization for supporting class-level isotropy}
\label{sec:z-score}

Z-score normalization is a standard preprocessing method \cite{montague2001relevance, singh2020investigating} that is particularly useful when the features in the dataset have different scales or units. By applying z-score normalization, the features are brought onto a common scale, which can improve the performance and stability of machine learning models, especially those that rely on distance calculations or gradient-based optimization \cite{imron2020improving, singh2021exploring}. For the proposed method, z-score normalization plays an additional role in promoting isotropic conditions at the class level. While z-score normalization can be applied to each class independently, this approach results in separate feature spaces for each class, making it incompatible for classifying new instances.

The goal of achieving isotropy is to have within-class covariance matrices that approximate a form proportional to the identity matrix. An isotropic covariance matrix is defined by having uniform diagonal elements and minimized off-diagonal elements. While the off-diagonal elements are controlled by maintaining low correlations between attributes, ensuring uniformity (or near uniformity) along the diagonal elements is also crucial for achieving the desired isotropic condition.

The diagonal elements of a within-class covariance matrix correspond to the variances of each individual attribute within the respective class. If the attributes in the original dataset have significantly different scales or units, these variances can differ vastly. In such cases, the covariance matrix would have non-uniform diagonal values, i.e., $\sigma_w^{(i)^2} \gg \sigma_w^{(j)^2}$ for some attribute combinations \( i, j \) and \( w = 1, 2 \), making it far from isotropic. Thus, even with low correlations between attributes, the resulting covariance matrix would not fit the form of a matrix proportional to the identity matrix.

Z-score normalization effectively addresses this issue by transforming each attribute to have a mean of zero and a standard deviation of one across the entire dataset. The following theorem formalizes the impact of z-score normalization on within-class variances:

\begin{theorem}[Within-class variance scaling through z-score normalization]\label{thm:Theorem3}
Consider the same setup as in Theorem \ref{thm:Theorem1}, with two sets of 1D data points representing two classes, denoted as $C_1$ and $C_2$, each containing $n$ samples. After applying z-score normalization to the entire dataset, the resulting within-class variances, denoted as $\sigma_1'^2$ and $\sigma_2'^2$, are given by the following relationships:

\[
\sigma_1'^2 = \frac{\sigma_1^2}{\sigma^2}, \quad \sigma_2'^2 = \frac{\sigma_2^2}{\sigma^2},
\]
where $\sigma_1^2$ and $\sigma_2^2$ are the original within-class variances, and $\sigma^2$ is the original overall variance of the combined dataset.
\end{theorem}

\begin{proof}
Let $\{x_1, x_2, \ldots, x_{2n}\}$ represent the original set of 1D data points. As defined in Theorem \ref{thm:Theorem1}, $C_1$ and $C_2$ are centered around their respective means, $\mu_1$ and $\mu_2$, with an overall mean $\mu$ for the combined dataset. The z-scored value for each data point $x_i$ is computed as:

\[
z_i = \frac{x_i - \mu}{\sigma}.
\]
Thus, the z-scored dataset is $\{z_1, z_2, \ldots, z_{2n}\}$, where each $z_i$ represents the standardized value of $x_i$.

The within-class variance for the z-scored data for class $C_1$ is calculated using the definition of variance:

\[
\sigma_1'^2 = \frac{1}{n-1} \sum_{i=1}^{n} \left( z_i - \bar{z}_1 \right)^2,
\]
where $\bar{z}_1$ is the mean of the z-scored data for class $C_1$, given by:

\[
\bar{z}_1 = \frac{\mu_1 - \mu}{\sigma}.
\]
Substituting $z_i = \frac{x_i - \mu}{\sigma}$, we get:

\[
\sigma_1'^2 = \frac{1}{n-1} \sum_{i=1}^{n} \left( \frac{x_i - \mu}{\sigma} - \frac{\mu_1 - \mu}{\sigma} \right)^2.
\]
Simplifying the expression, we obtain:

\[
\sigma_1'^2 = \frac{1}{\sigma^2} \cdot \frac{1}{n-1} \sum_{i=1}^{n} \left( (x_i - \mu_1)^2 \right),
\]
Thus, we have:

\[
\sigma_1'^2 = \frac{\sigma_1^2}{\sigma^2}.
\]

Similarly, for class $C_2$, the within-class variance for the z-scored data is given by:

\[
\sigma_2'^2 = \frac{1}{\sigma^2} \cdot \frac{1}{n-1} \sum_{j=n+1}^{2n} \left( (x_j - \mu_2)^2 \right) = \frac{\sigma_2^2}{\sigma^2}.
\]
This completes the proof.
\end{proof}

Theorem \ref{thm:Theorem3} establishes that after z-score normalization, the within-class variance for a particular attribute is normalized relative to the overall variance of that attribute in the combined dataset. This normalization ensures that the variances for different attributes within the same class become relatively uniform, i.e., $\sigma_w'^{(i)^2} \approx \sigma_w'^{(j)^2}$ for all attribute combinations \( i, j \) and \( w = 1, 2 \), bounded within \([0, 1]\), thereby approximating isotropy within each class.

The degree to which a within-class covariance matrix approximate isotropy depends on how similar the relative variances are across different attributes within the respective class. While perfect uniformity may not be achieved, the application of z-score normalization significantly reduces the disparity between the within-class variances, leading to covariance matrices that better fit the desired isotropic form.

Thus, combining z-score normalization and maintaining low correlations between attributes helps produce within-class covariance matrices that are closer to the ideal isotropic condition, supporting the validity of Theorem \ref{thm:Theorem1} in higher dimensions.

\section{Complexity analysis}

The overall time complexity of our method encompasses several tasks, including covariance and Hessian matrix computations ($\mathcal{O}(N \cdot F^2)$), eigenanalysis of these matrices ($\mathcal{O}(F^3)$), selection of eigenvectors corresponding to leading eigenvalues ($\mathcal{O}(1)$), and data projection into the combined space ($\mathcal{O}(N \cdot F)$). Here, $N$ signifies the number of data points in the training set, while $F$ represents the number of features per data point. It is essential to note that this analysis considers our method's computational demands under the assumption that a pre-existing DNN is in place.

Comparatively, the time complexity of LDA is primarily influenced by the calculation of both within-class and between-class scatter matrices ($\mathcal{O}(N \cdot F^2)$) and subsequent eigenanalysis ($\mathcal{O}(F^3)$). In both methods, the dominant complexity factor is either the matrix computations ($\mathcal{O}(N \cdot F^2)$) when there are more examples than features or the eigenanalysis step ($\mathcal{O}(F^3)$) otherwise.

This analysis reveals that our proposed method exhibits a comparable computational profile to the method under improvement. Therefore, our approach offers enhanced class separability and interpretability without significantly increasing computational demands, making it a practical choice for real-world applications.

\section{Experiment}
\label{headings}

\subsection{Assessed dimensionality reduction techniques}

In our evaluation, we compare nine distinct dimensionality reduction and data projection techniques, each offering unique insights: principal component analysis (PCA), kernel PCA (KPCA) \citep{scholkopf1997kernel}, Hessian, uniform manifold approximation and projection (UMAP) \citep{mcinnes2018umap}, locally linear embedding (LLE) \citep{roweis2000nonlinear}, linear optimal low-rank projection (LOL) \citep{vogelstein2021supervised}, linear discriminant analysis (LDA), kernel discriminant analysis (KDA) \citep{mika1999fisher}, and the proposed approach. PCA involves projection onto the primary two covariance eigenvectors, i.e., it applies KPCA with a linear kernel. Similarly, LDA employs KDA with a linear kernel. The Hessian method projects data onto the leading two Hessian eigenvectors, capturing the most significant directions of curvature in the loss landscape. While PCA and the Hessian method rely exclusively on the top two eigenvectors of their respective matrices, the proposed method synergistically integrates the top eigenvector from both the covariance and Hessian matrices, as specified by Eqs(\ref{eq:a}, \ref{eq:b}). Notably, for both KPCA and KDA, the best kernels (linear, polynomial, RBF, sigmoid, or cosine similarity) and their associated parameters (kernel coefficient or degree) are determined using grid search tailored to each dataset. 

\subsection{Class separability assessment via linear SVMs}

The nine dimensionality reduction methods we assess are designed to transform high-dimensional data into more manageable 1D or 2D spaces while simultaneously enhancing or preserving class separability. Linear SVMs enable us to create decision boundaries that are readily visualized in these 1D or 2D spaces. Thus, we utilize SVMs with a linear kernel to evaluate and visualize the extent of class separability achieved through the projections.

In addition to the visualization of decision boundaries, we employ a comprehensive set of evaluation metrics to quantify the performance of the dimensionality reduction methods. These metrics include the F1 score, which measures the balance between precision and recall, the Area Under the Receiver Operating Characteristic Curve (AUC ROC), which assesses the model's ability to distinguish between positive and negative classes, and Cohen's kappa, a statistic that gauges the agreement between the predicted and actual class labels. 

\subsection{Datasets}

We will conduct our assessment of the nine distinct methods on the following datasets:

\textbf{Widely recognized benchmark datasets.} We evaluate our approach using three widely recognized benchmark datasets for binary classification: the Wisconsin breast cancer dataset \citep{street1993nuclear}, the heart disease dataset \citep{detrano1989international}, and the Pima Indians diabetes dataset \citep{smith1988using}. Prior to applying various dimensionality reduction methods, we enact standard data preprocessing techniques on the original datasets, including handling of missing data, one-hot encoding for categorical variables, and data normalization. Specifically, we apply z-score normalization to ensure the data has zero mean and unit variance, supporting the requirement for the data to be isotropic as suggested in Section \ref{ideal_condition}.

\textbf{Neural spike train dataset.} A spike train is a sequence of action potentials (spikes) emitted by a neuron over time. The neural spike train dataset used in this research consists of recordings from rat's neurons during drug application from a multi-electrode array \citep{tsai2015high, heuschkel2002three}. The data, comprising 221 records, represents the final dataset after all preprocessing steps suggested by \citep{lazarevich2023spikebench}. Each record contains 15 time-series features extracted using the `tsfresh` package \citep{christ2018time} and 1 class attribute indicating whether the neuron is non-responsive (0) or responsive (1). Exploratory data analysis revealed an imbalance in the dataset, with 190 non-responsive neurons (86\%) and 31 responsive neurons (14\%). 

To provide an unbiased assessment of our method's performance, we conduct extensive experiments on the datasets with 10-fold cross-validation with the exception of the neural spike train dataset. Due to the limited number of positive cases in the neural spike train dataset, we performed 5-fold cross-validation to ensure a reasonable sample size for validation.

\subsection{Results}

\begin{figure*}
  \centering
  \includegraphics[width=14cm]{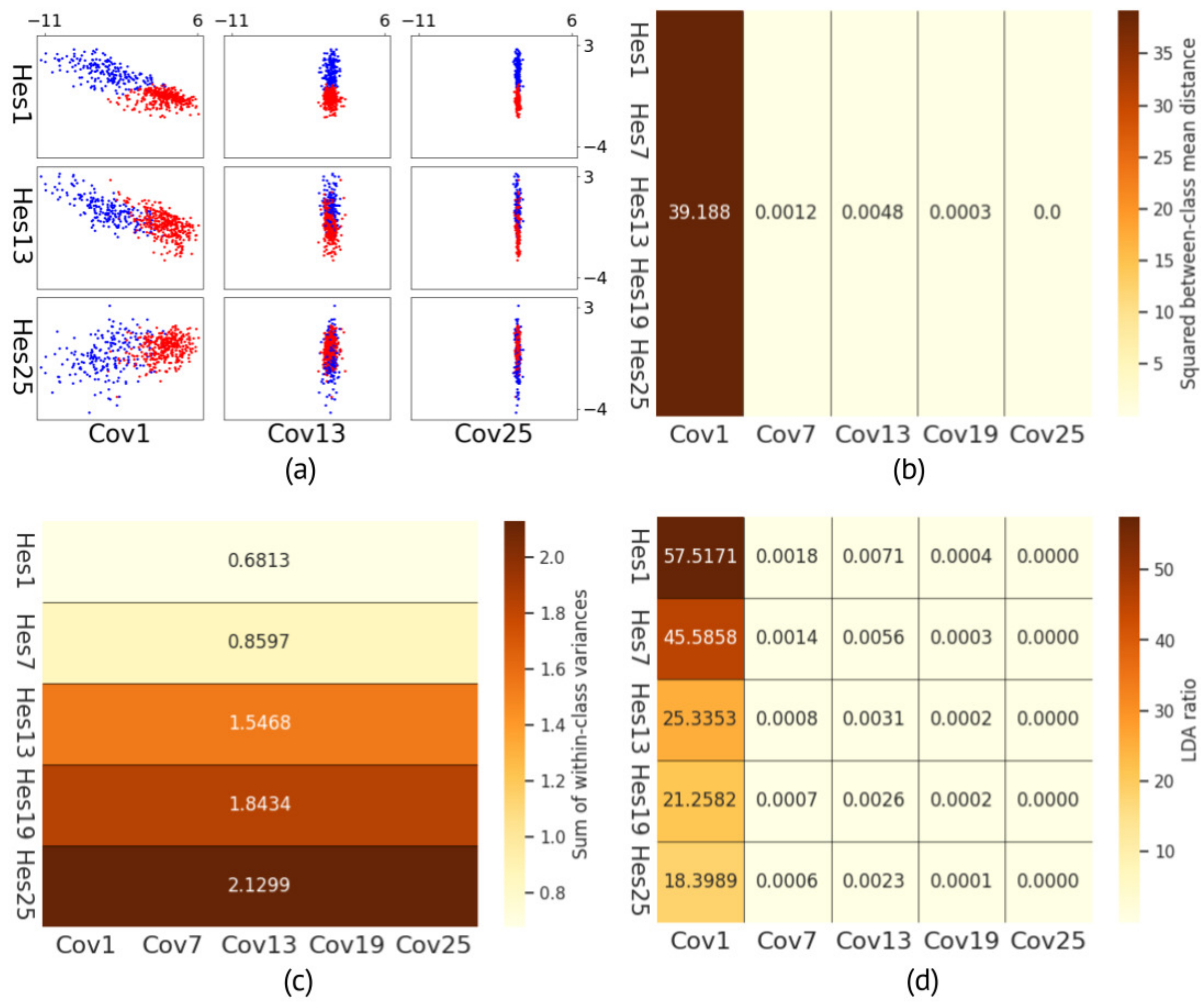}
  \caption{\textbf{Projection of the Wisconsin breast cancer data into different combined spaces of the covariance and Hessian eigenvectors.} \textbf{(a)} Nine selected projection plots, each representing data projected onto a distinct space created by combining the first three covariance and first three Hessian eigenvectors. \textbf{(b)} Heatmap showing the squared between-class mean distance for projections onto varying combinations of covariance and Hessian eigenvectors. The heatmap demonstrates that the values remain constant vertically across different Hessian eigenvectors, while exhibiting a noticeable descending order horizontally, aligning with the descending order of the variance. These results essentially concretize our formal premise, empirically validating the linear relationship described in Eq(\ref{eq:1}) between the variance and the squared between-class mean distance. \textbf{(c)} Heatmap showing the sum of within-class variances for projections onto different combinations of covariance and Hessian eigenvectors. The values remain constant horizontally across different covariance eigenvectors but exhibit a clear ascending order vertically, aligning with the descending order of the Hessian. These empirical results validate the negative correlation between the Hessian and the within-class variance described in Eq(\ref{eq:2}) within the framework of our theoretical foundation. \textbf{(d)} Heatmap displaying the LDA ratio, representing the ratio between the squared between-class mean distances presented in (b) and the corresponding within-class variances shown in (c). The highest LDA ratio is observed for the combination of the first Hessian eigenvector with the first covariance eigenvector. Notably, a general descending pattern is observed both horizontally and vertically across different combinations, indicating that both covariance eigenanalysis (represented along the horizontal direction) and Hessian eigenanalysis (represented along the vertical direction) equally contribute to the class separability (represented by the LDA ratio).}
  \label{fig1}
\end{figure*}

The results depicted in Figure \ref{fig1}, which pertain to the WBCD dataset, along with the corresponding findings presented in \ref{more-results} for the other three datasets, collectively reaffirm the consistency and robustness of our proposed approach across diverse datasets. The heatmaps shown in Figure \ref{fig1}(b) and \ref{more-results} consistently demonstrate the reduction in squared between-class mean distance as more covariance eigenvectors are incorporated, aligning with our established theoretical framework (Eq(\ref{eq:1})). Furthermore, Figure \ref{fig1}(c) and its counterparts in the appendix reveal the ascending order of within-class variances, in sync with the descending order of the Hessian, supporting our theoretical foundation (Eq(\ref{eq:2})). Furthermore, Figure \ref{fig1}(d) and its counterparts illustrate the LDA ratio, emphasizing the equal contributions of covariance and Hessian eigenanalyses to class separability. Importantly, the results also imply that the highest LDA ratio is observed for the combination of the first Hessian eigenvector with the first covariance eigenvector. This observation underscores the significance of projecting data onto the combined space of these primary eigenvectors (as outlined in Eqs (\ref{eq:a}, \ref{eq:b})), forming the core of our proposed method. These consistent empirical results across multiple datasets not only validate our theoretical premises but also endorse the effectiveness of our proposed method in optimizing class separability.

The evaluation results in Figure \ref{fig2} compare the performance of different data projection methods, including PCA, KPCA, Hessian, UMAP, LLE, LOL, LDA, KDA, and the proposed method, using 5- or 10-fold cross-validation. It is essential to note that we introduced nonlinearity to the comparative analysis by utilizing distinct kernels determined through grid search in KPCA and KDA for each dataset. Notably, our proposed method consistently outperforms all others, securing the highest average scores across all datasets and evaluation metrics. This consistent superiority of our approach implies its potential as a valuable tool for improving classification performance in various domains, further highlighting its promise as a robust and effective method for dimensionality reduction and data projection.

\begin{figure*}[htbp]
  \centering
  \includegraphics[width=14cm]{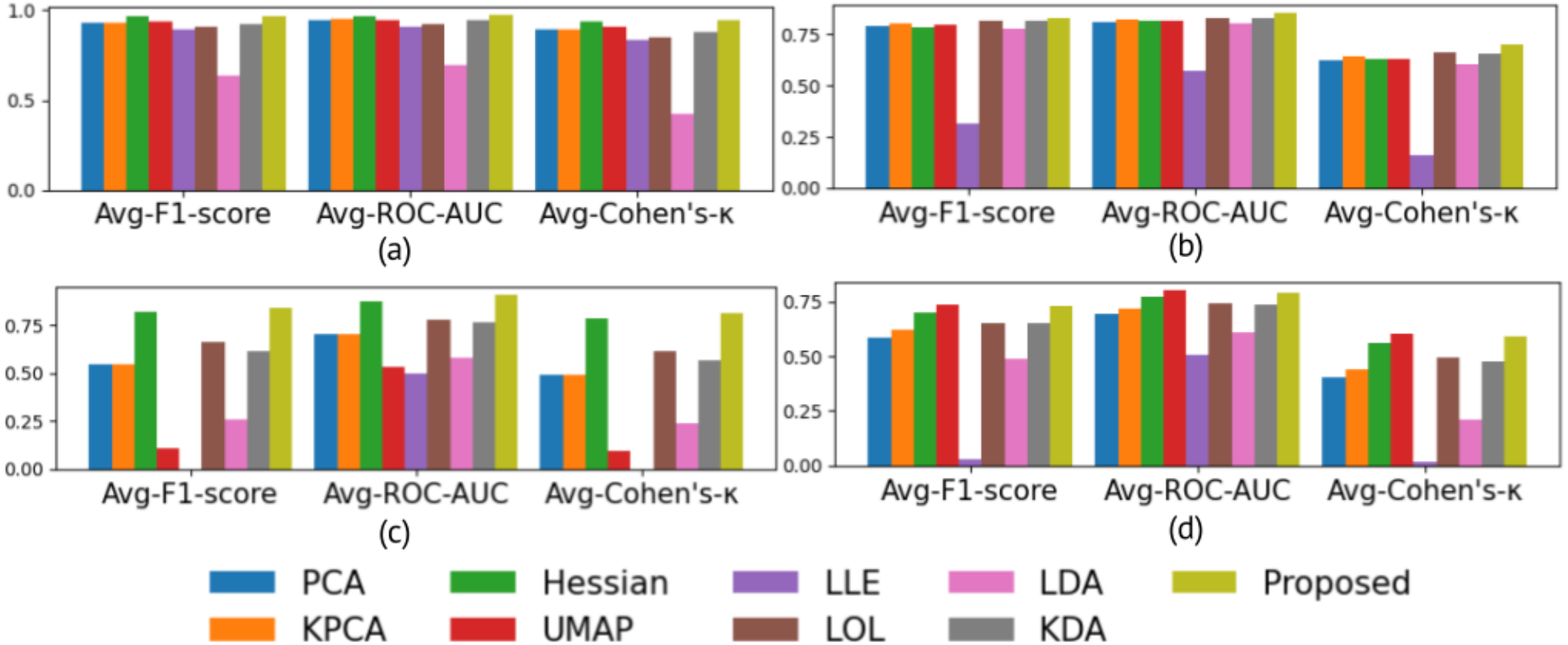}
  \caption{\textbf{Performance comparison of data projection methods using cross-validation on four distinct datasets: (a) WBCD, (b) heart disease, (c) neural spike train, and (d) Pima Indians diabetes datasets.} This figure presents the average F1 score, ROC AUC, and Cohen's Kappa values obtained through 5- or 10-fold cross-validation for nine data projection techniques: PCA, KPCA (cosine similarity for WBCD, polynomial kernel with degree=3 for Heart, linear kernel for Neural, cosine similarity for Pima), Hessian, UMAP, LLE, LOL, LDA, KDA (cosine similarity for WBCD, RBF kernel with coefficient=0.01 for Heart, sigmoid kernel with coefficient=2 for Neural, cosine similarity for Pima), and the proposed method. Notably, the proposed method consistently outperforms all other techniques, achieving the highest scores across all evaluation metrics.}
  \label{fig2}
\end{figure*}

\section{Illustrative case study: WBCD dataset}

In this section, we illustrate the practical application of the proposed method using the WBCD dataset. The WBCD dataset comprises 30 predictor attributes for diagnosing breast cancer. This section presents detailed analyses and visualization of the WBCD dataset using our method and other established data projection techniques.

We first examine the separability of classes in the WBCD dataset using basic pairwise attribute projections. Figure \ref{fig:pairplot} shows a pairplot of the first four predictor attributes, colored by class labels. This pairplot projects the 30-dimensional data into 1D and 2D spaces between individual attributes and attribute pairs. As can be visually observed from the figure, these simple projections do not provide good separability between the two classes. This is because these projections are not optimized for class separability. In contrast, established projection methods such as PCA, LDA, and UMAP, as well as our proposed method, are specifically designed and optimized to enhance class separability.

\begin{figure*}
  \centering
  \includegraphics[width=16cm]{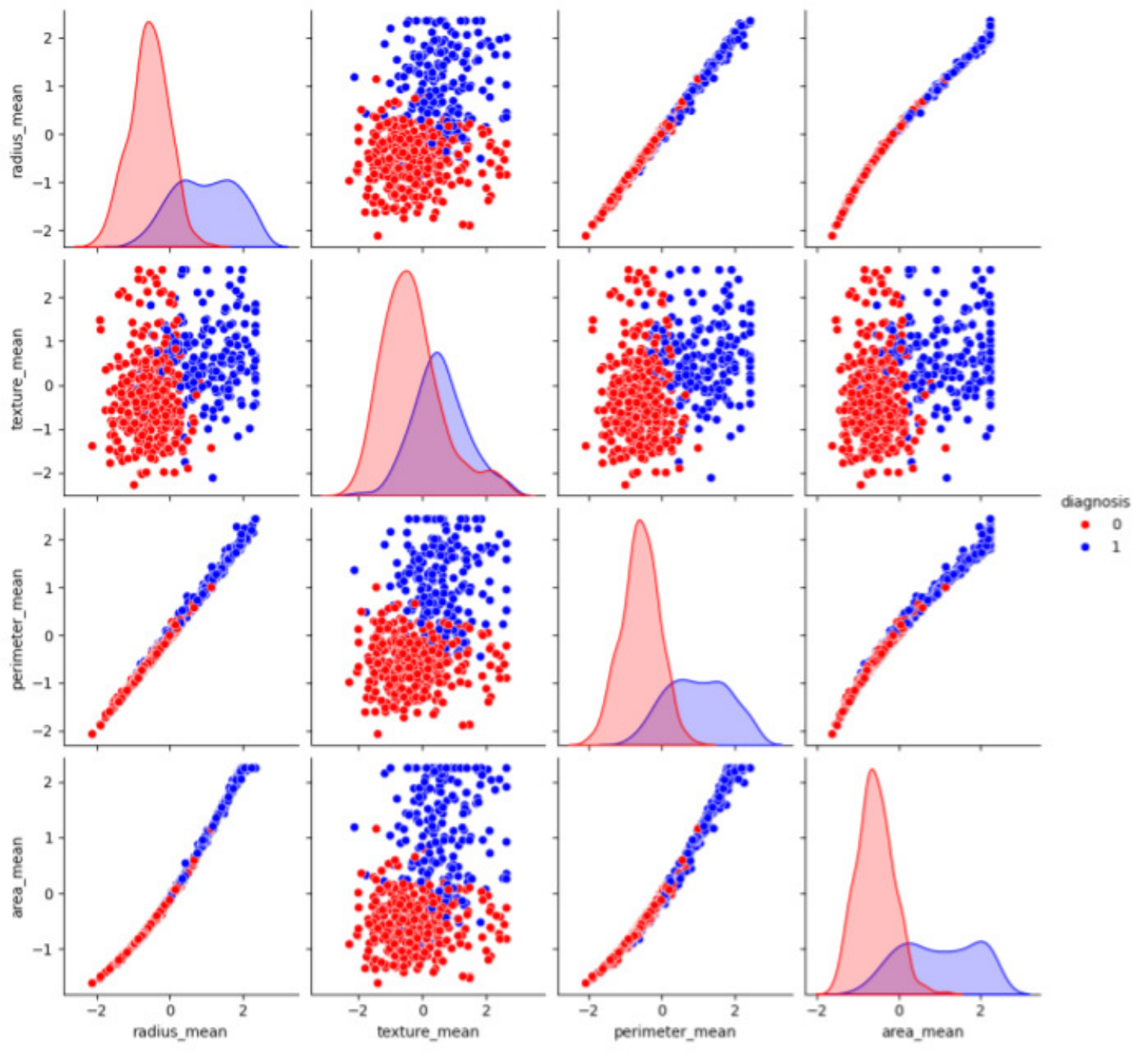}
  \caption{\textbf{Pairplot of the first four predictor attributes in the WBCD dataset, colored by class labels.} This pairplot presents only four of the 30 predictor attributes, projecting the 30-dimensional data into 1D and 2D spaces between individual attributes and attribute pairs. The diagonal plots display the distribution of each attribute individually using kernel density estimation (KDE) curves, where separation is assessed based on the overlap of the KDE curves for each class. The non-diagonal plots show scatter plots of pairs of attributes, where separation between classes is evaluated based on the extent of overlap or spread of points corresponding to different class labels. The visualizations show that these projections do not provide good separability between the two classes.}
  \label{fig:pairplot}
\end{figure*}

\begin{figure*}
  \centering
  \includegraphics[width=11cm]{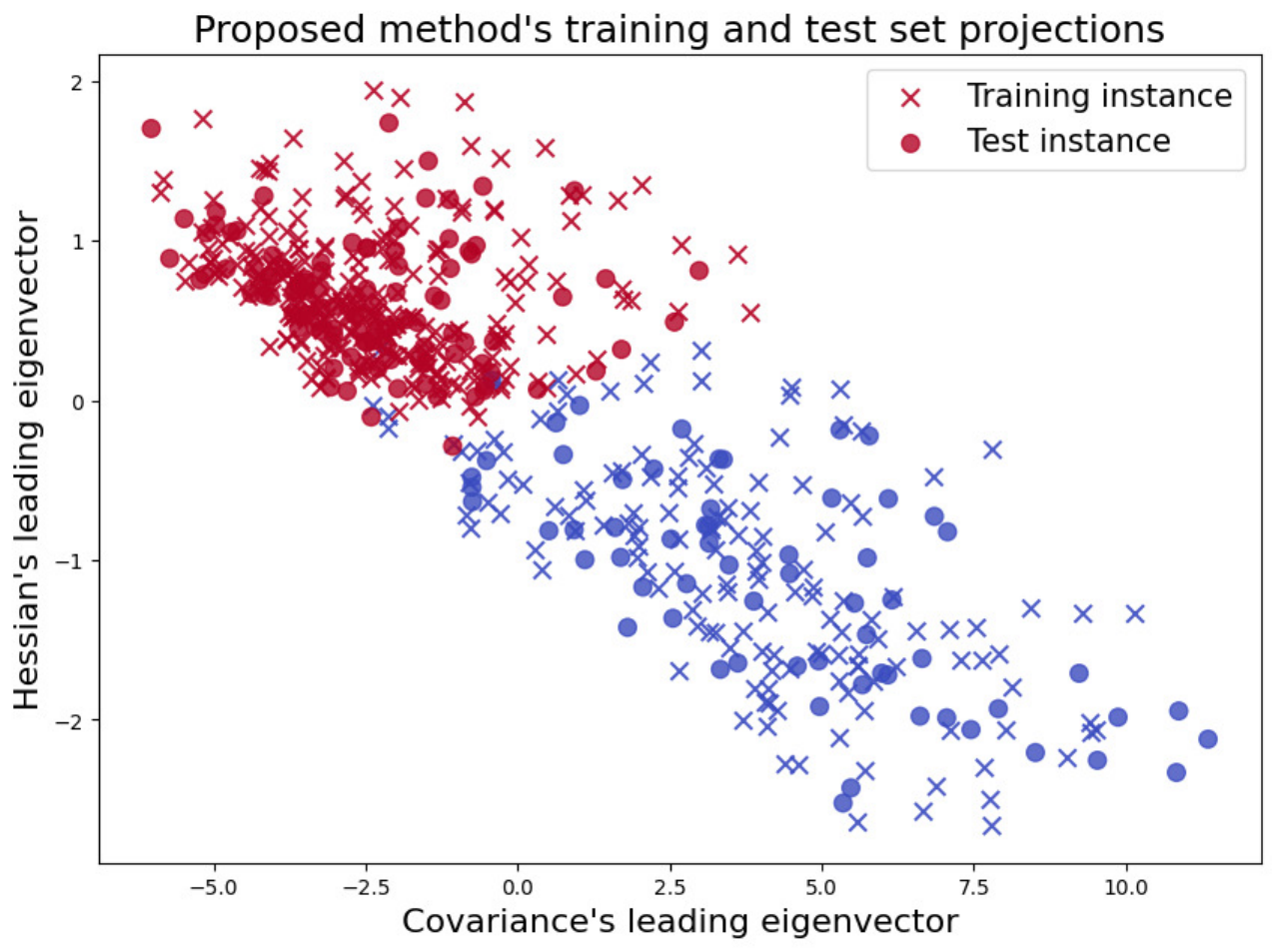}
  \caption{\textbf{Training set and test set projections on the WBCD dataset using the proposed method.} The proposed method projects the high-dimensional data into the leading eigendirections of the covariance and Hessian matrices. The training set projection illustrates how the training data is separated using the proposed method, while the test set projection shows the separation of new unseen data. Visually, the projections indicate good separability between the two classes, demonstrating the effectiveness of the proposed method.}
  \label{fig:projection}
\end{figure*}

Our proposed method leverages the strengths of both the covariance matrix and the Hessian matrix to achieve optimal class separability. The covariance matrix, computed unsupervisedly from the training set using Eq. \ref{eq:cov}, is a square and symmetric matrix that captures the statistical spread and relationships between predictor attributes. The Hessian matrix, evaluated on a deep learning model trained on the same training set using Eq. \ref{eq:hess}, is another square and symmetric matrix that reflects the sensitivity and curvature of the loss function with respect to the model parameters. The DNN test scores are as follows: accuracy: 0.9883, F1 score: 0.9841, AUC ROC: 0.9917, Cohen Kappa score: 0.9749, geometric mean score: 0.9874.

By performing eigenanalysis on the covariance and Hessian matrices, we obtain leading eigenvectors that represent the principal directions of variance and the sharpest curvature, respectively. These eigendirections are crucial for projecting the high-dimensional data into a 2D space that enhances class separability by maximizing the between-class mean distance and minimizing the within-class variances, as described in Theorem \ref{thm:Theorem1} and Theorem \ref{thm:Theorem2}. We project the data into the combined space of the two eigendirections using Eq. \ref{eq:a} and Eq. \ref{eq:b}. Figure \ref{fig:projection} visually demonstrates the effectiveness of the proposed method by showing the 2D projections of the WBCD dataset. The training set projection highlights the observed visual separation achieved on the training data, while the test set projection confirms good visual separability for new unseen data, underscoring the method's potential to generalize well. The visual separability observed here will be quantified later using SVM (Figures \ref{fig:train_comparison} and \ref{fig:test_comparison}).

Our proposed method not only enhances class separability but also facilitates interpretability and explainability of the model. Figure \ref{fig:param_contributions} presents the parameter contributions to the leading eigenvectors of the covariance and Hessian matrices for the WBCD dataset. The contributions are calculated as the absolute values of the elements in the first eigenvector of each matrix. Parameters are sorted by their absolute contributions, and the horizontal bar plots display these sorted contributions for each parameter. The first plot indicates which attributes contribute the most to maximizing the between-class mean distance, providing insights into the key factors that drive class separation. The second plot shows which attributes contribute the most to minimizing the within-class variances, highlighting the factors that help maintain compact and distinct clusters. These visualizations enable a deeper understanding of the model's decision-making process, making it easier to interpret and explain the results.

\begin{figure*}
  \centering
  \includegraphics[width=16cm]{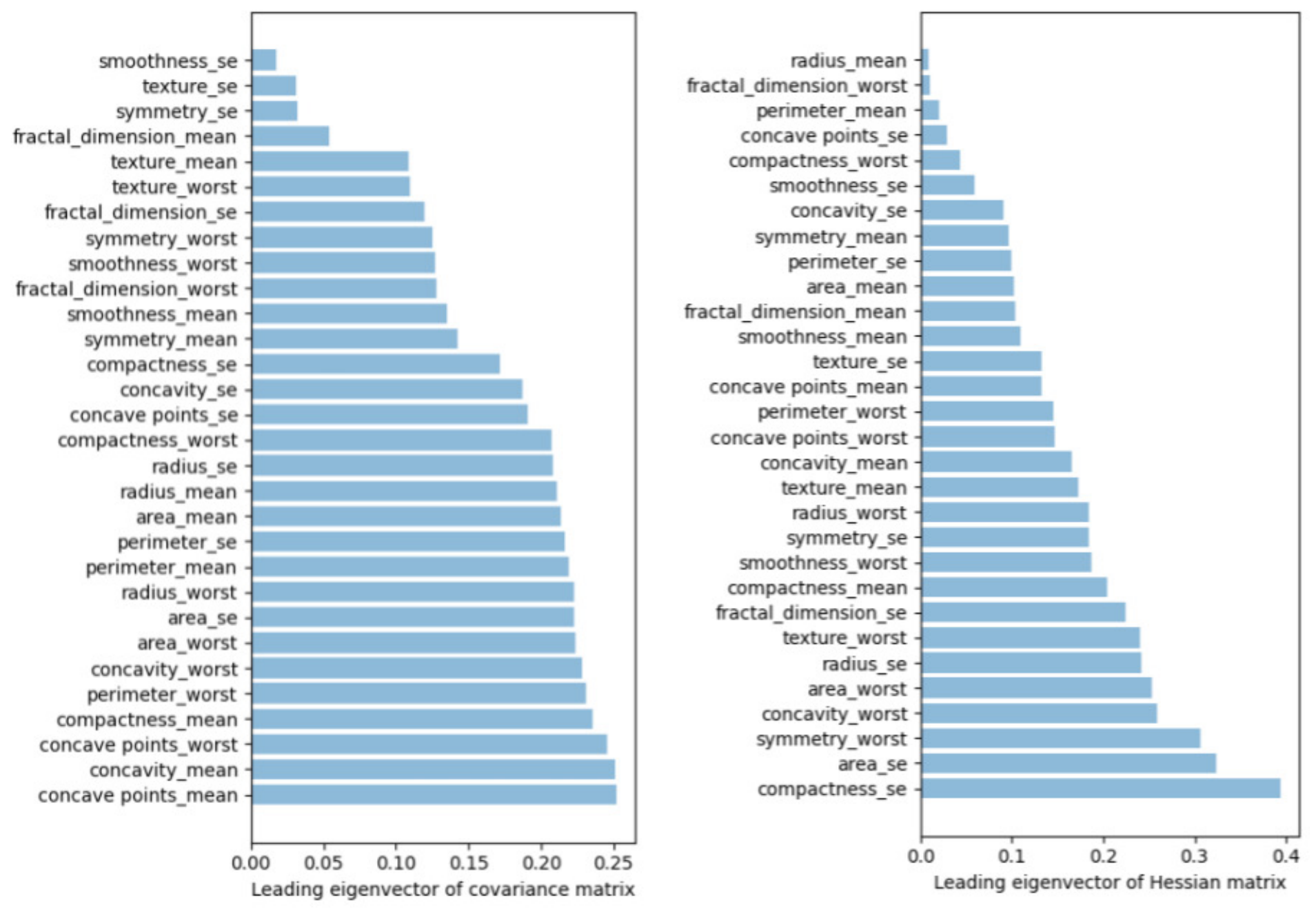}
  \caption{\textbf{Parameter contributions to the leading eigenvectors of the covariance and Hessian matrices for the WBCD dataset.} The contributions are calculated as the absolute values of the elements in the first eigenvector of the covariance and Hessian matrices, respectively. Parameters are sorted by their absolute contributions, and the horizontal bar plots display these sorted contributions for each parameter. The first plot indicates which attributes contribute the most to maximizing the between-class mean distance, while the second plot indicates which attributes contribute the most to minimizing the within-class variances.}
  \label{fig:param_contributions}
\end{figure*}

\begin{figure*}
  \centering
  \includegraphics[width=16cm]{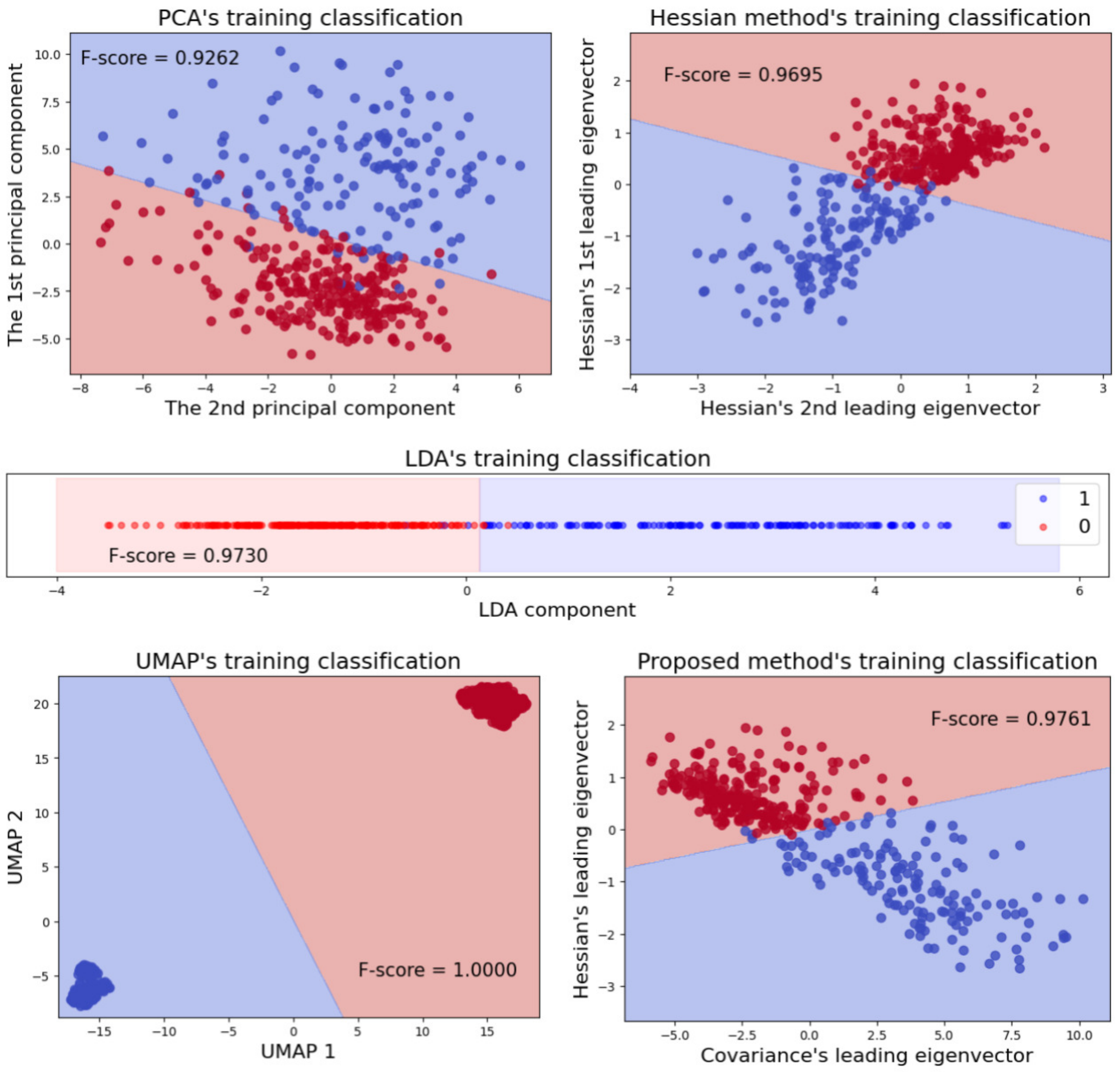}
  \caption{\textbf{Training classification visualizations using various dimensionality reduction methods and linear SVM on the WBCD dataset.} The plots show the separation of training data points using PCA, the Hessian method, LDA, UMAP, and the proposed method. Each subplot includes the separating SVM hyperplane and the corresponding F-score, indicating the classification performance of each method. UMAP achieves the highest performance with perfect separation, while the proposed method shows strong performance, surpassing the other methods.}
  \label{fig:train_comparison}
\end{figure*}

\begin{figure*}
  \centering
  \includegraphics[width=16cm]{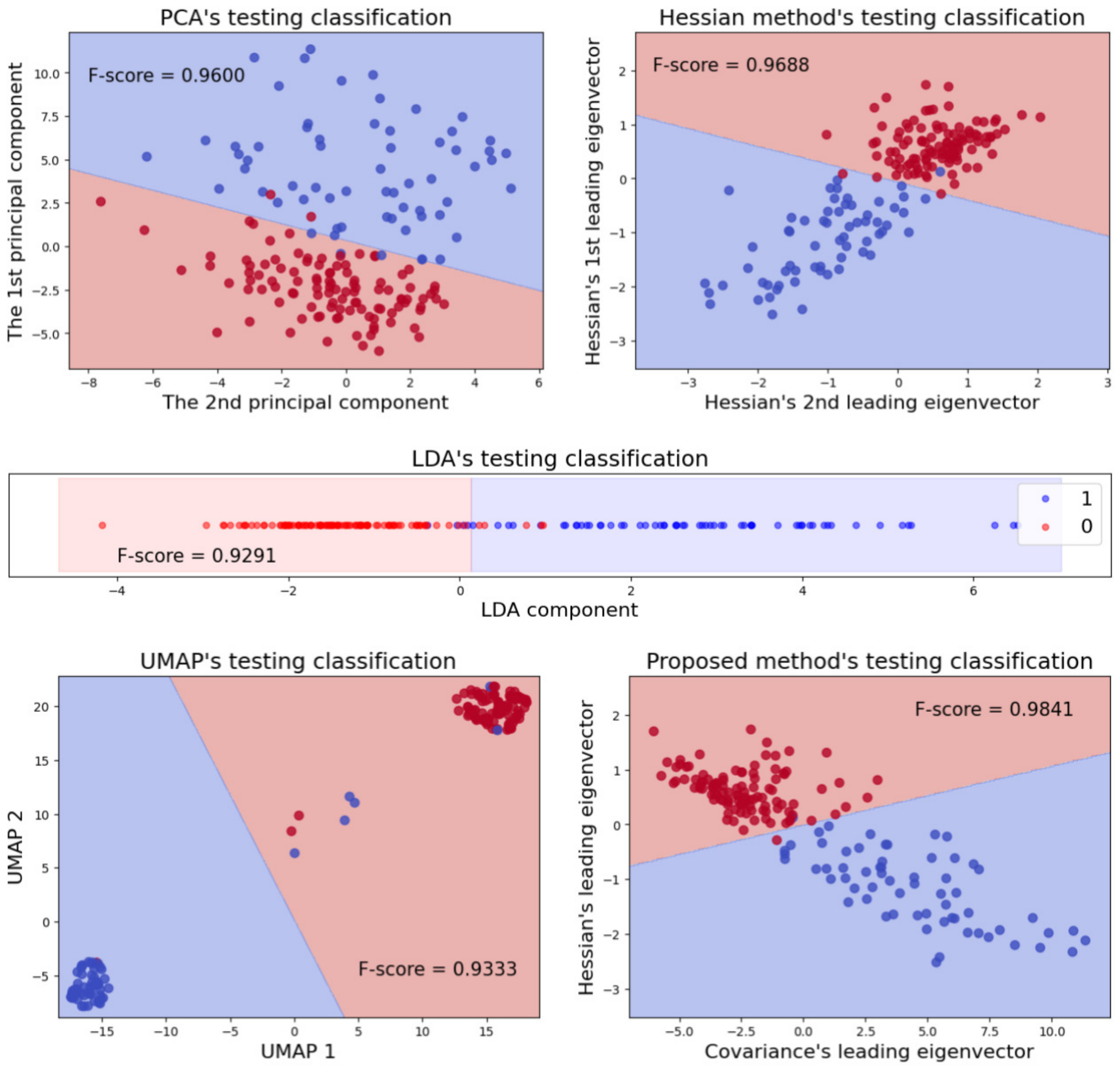}
  \caption{\textbf{Testing classification visualizations using various dimensionality reduction methods and SVM on the WBCD dataset.} The SVM hyperplanes trained on the training set projected using PCA, the Hessian method, LDA, UMAP, and the proposed method, as shown in Figure \ref{fig:train_comparison}, are then applied to new unseen data points. The plots show the separation of these new data points. Each subplot includes the separating SVM hyperplane and the corresponding F-score, indicating the classification performance of each method on the test set.  On new data, the proposed method outperforms all other methods, including UMAP, which had achieved the best performance on the training set. This demonstrates the superior generalization capability of the proposed method.}
  \label{fig:test_comparison}
\end{figure*}

While Figures \ref{fig:projection} illustrate the visual separations achieved by the proposed method, Figures \ref{fig:train_comparison} and \ref{fig:test_comparison} quantify their degrees of separation. These figures present the classification performance using various dimensionality reduction methods, including PCA, the Hessian method, LDA, UMAP, and the proposed method, evaluated with linear SVM on the WBCD dataset. Each subplot shows the separating SVM hyperplane along with the corresponding F-score, providing a quantitative measure of the classification performance. Figure \ref{fig:train_comparison} highlights the separation achieved on the training data, where UMAP achieves the highest performance with perfect separation, while the proposed method demonstrates strong performance, surpassing other methods. Figure \ref{fig:test_comparison} shows the performance on new unseen test data, where the proposed method outperforms all other methods, including UMAP, demonstrating its superior generalization capability.

\section{Discussion}

Our work offers a valuable theoretical perspective and a practical method, emphasizing the utility of simplicity in achieving noteworthy outcomes. We present a theoretical framework that elucidates a significant relationship between covariance and Hessian matrices. This framework, supported by our formal proof, connects covariance eigenanalysis with the first LDA criterion (the concept of \textit{separation}) and Hessian eigenanalysis with the second (the concept of \textit{compactness}). While this theoretical clarity is compelling, the effectiveness of our method, demonstrated across diverse datasets, is contingent upon relatively ideal conditions as outlined in Section \ref{ideal_condition}. These conditions include assumptions about data isotropy, specifically the dataset subsets corresponding to each class being approximately isotropic around their respective means, and the dominance of leading eigenvalues, which are crucial for the method's success in practical applications.

The experiments conducted on the WBCD dataset illustrate these points. As shown in Figure \ref{fig:abs_cov_matrices}, the absolute covariance matrices for each class in the WBCD dataset reveal a pattern that approximates the isotropic condition. The relatively uniform higher values along the diagonal and generally lower values in the off-diagonal elements indicate that the covariance matrices are nearly proportional to the identity matrix. In contrast, as shown in \ref{unnormalized_cov}, the absolute within-class covariance matrices for the unnormalized WBCD dataset exhibit a significant dominance of a few components that overshadow the other components, resulting in a structure that is far from isotropic and a far less suitable foundation for achieving effective class separability. This highlights the importance of normalization as discussed in Subsection \ref{sec:z-score}. The WBCD dataset being normalized (as with the other datasets used in this study) helps approximate isotropy, as does the dataset being a well-structured dataset with low redundancy among attributes. This near-isotropy supports the method's effectiveness in maximizing the between-class mean distance when projecting the multidimensional data into the leading eigenvector of the covariance matrix, as suggested by Theorem \ref{thm:Theorem1}. Notably, the superior performance of the proposed method under these near-isotropic conditions demonstrates its robustness and ability to achieve optimal results even when perfect isotropy is not strictly present.

Moreover, the eigenspectra displayed in Figure \ref{fig:eigenspectra} highlight the dominance of the first eigenvalues in both the covariance and Hessian matrices. The eigenvalues are spread over several orders of magnitude on a logarithmic scale, often with approximately uniform spacing. Notably, the first eigenvalue for both matrices is situated on a different decade from the second and the subsequent eigenvalues indicating a significant separation in magnitude. This characteristic, especially pronounced in the Hessian matrix, is indicative of a "sloppy" model \citep{waterfall2006sloppy, transtrum2010nonlinear, transtrum2011geometry, transtrum2015perspective, machta2013parameter, raman2017delineating, hartoyo2019parameter}, where a few stiff directions dominate. Similarly, in the covariance matrix, this pattern suggests a dominant underlying structure or trend in the data. This distribution ensures that the leading eigenvector captures a significant portion of the variance or curvature, crucial for the proposed method's effectiveness in distinguishing between classes. 

The dominance of the leading eigenvectors in capturing the most significant information is further reflected in the comparative performance results shown in Figure \ref{fig:dnn_performances}. The results show that while Hesssian method and the proposed method sometimes outperform the standard DNNs, in other cases, DNNs perform better. This variability arises because DNNs utilize the entire feature space, while the projection methods operate within a reduced eigenspace using only a subset of the most informative eigendirections. When the omitted directions contain meaningful information, the standard DNNs excel; conversely, when the excluded features are primarily noise, the projection methods yield superior results. Importantly, in all scenarios, the proposed method consistently achieves strong performance, highlighting its ability to leverage the dominant eigenvectors and capture significant information from both the covariance and Hessian matrices.

\begin{figure*}
  \centering
  \includegraphics[width=16cm]{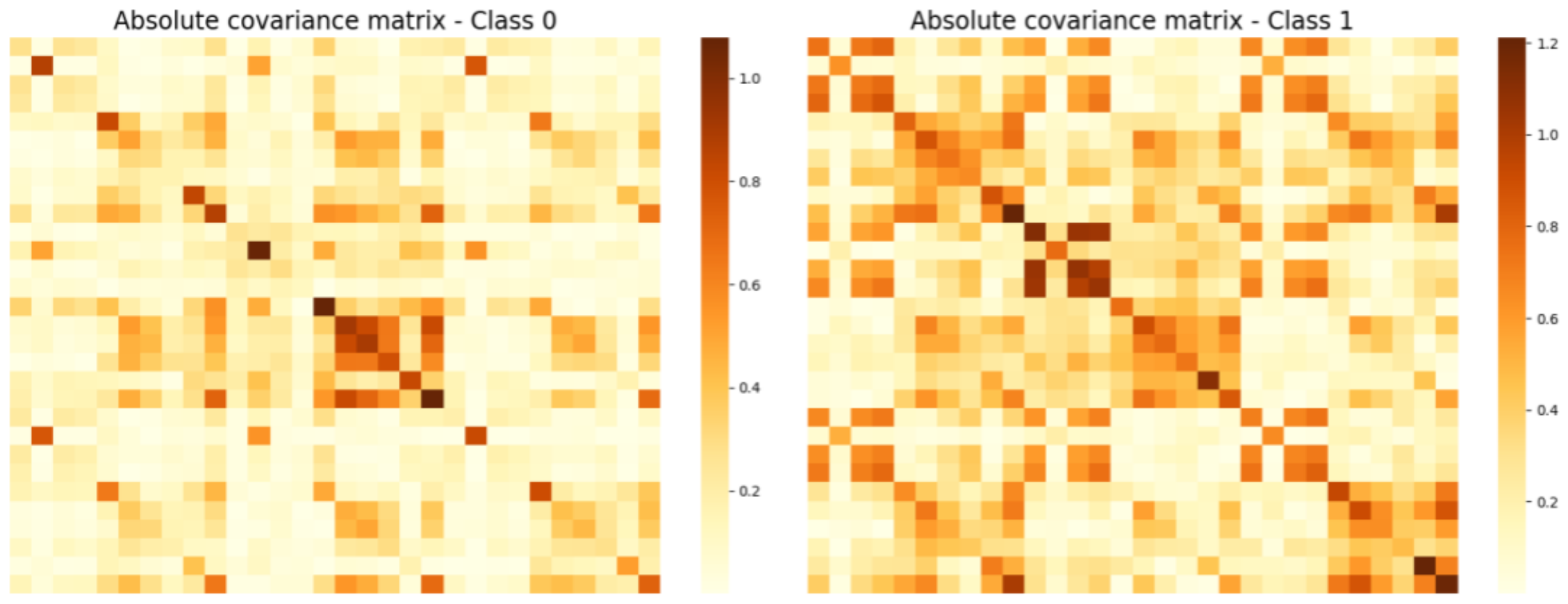}
  \caption{\textbf{Absolute covariance matrices for (normalized) WBCD dataset by class.} The figure presents the absolute values of the covariance matrices for the WBCD dataset, calculated separately for each class. These matrices highlight the strength of linear relationships between features within each class. In these absolute covariance matrices, the approximate isotropic condition is evidenced by the relatively uniform higher values along the diagonal and the generally lower values in the off-diagonal elements. Specifically, for Class 0, the average diagonal value is 0.5655, and the average non-diagonal value is 0.1566; for Class 1, these values are 0.8503 and 0.2828, respectively. These characteristics, though not starkly isotropic, approximate the conditions where the covariance matrices are proportional to the identity matrix. This approximation helps ensure the applicability of Theorem \ref{thm:Theorem1} in higher dimensions, as evidenced by the superior performance of the proposed method in practice.}
  \label{fig:abs_cov_matrices}
\end{figure*}

\begin{figure*}
  \centering
  \includegraphics[width=10cm]{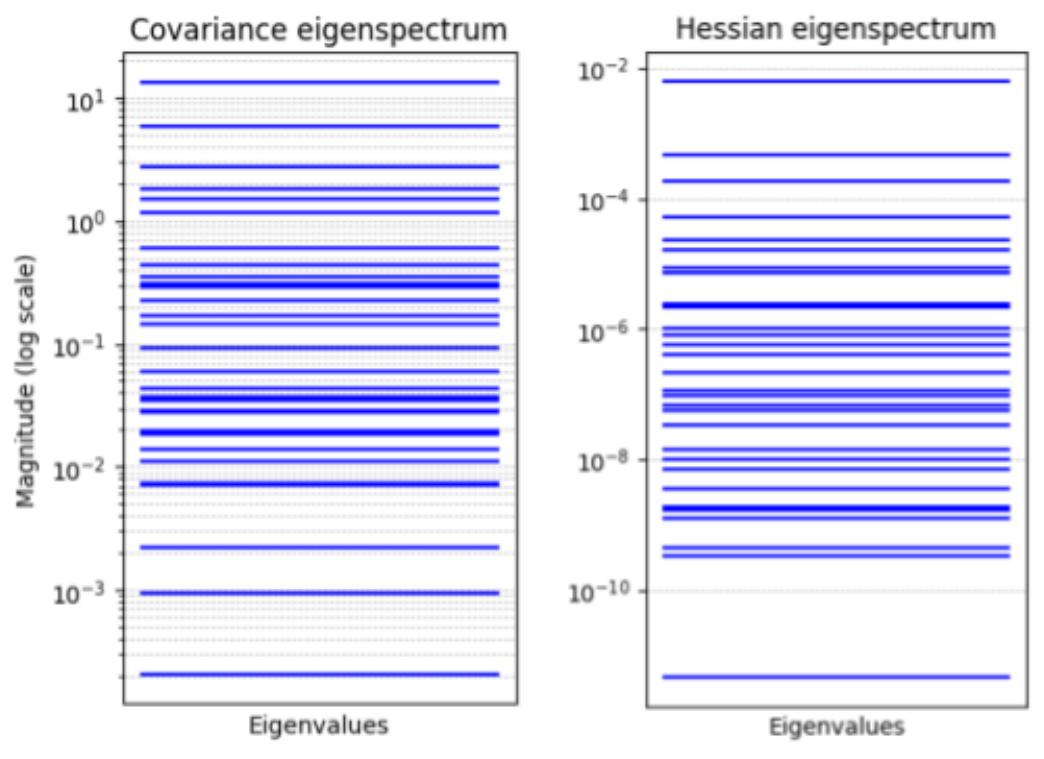}
  \caption{\textbf{Eigenspectra of Covariance and Hessian Matrices for the WBCD Dataset.} The plots display the eigenspectra of the covariance (left) and Hessian (right) matrices on a logarithmic scale. In both spectra, it is visible that the eigenvalues are spread over several orders of magnitude, often with approximately uniform spacing between them. Remarkably, the first eigenvalue in both matrices is positioned an order of magnitude above the second and subsequent eigenvalues, indicating a substantial difference in scale. This distribution suggests that the leading eigenvector captures a substantial portion of the variance or curvature, which is crucial for the effectiveness of the proposed method.}
  \label{fig:eigenspectra}
\end{figure*}

\begin{figure*}
  \centering
  \includegraphics[width=14cm]{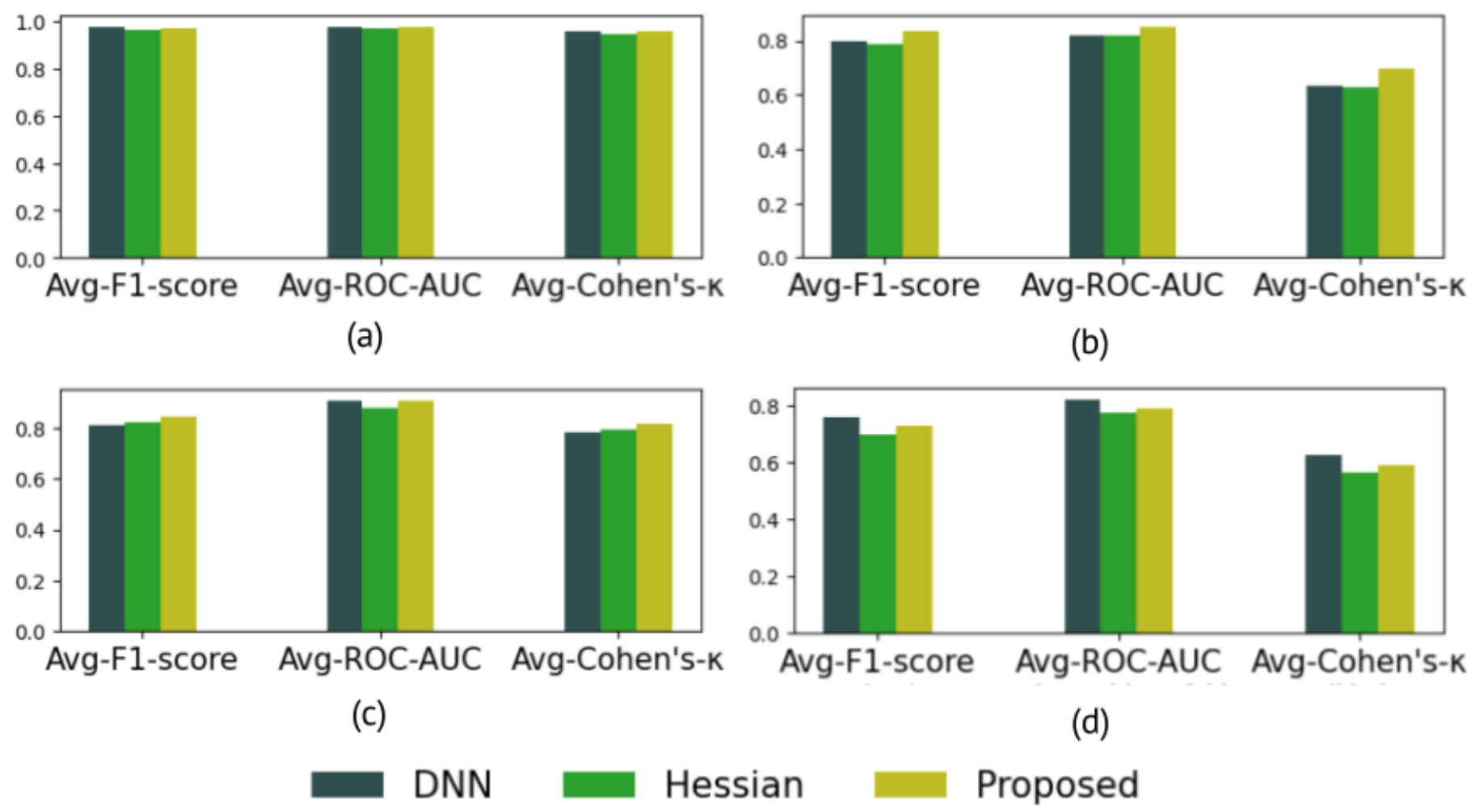}
  \caption{\textbf{Performance comparison of DNN and the two DNN-based data projection methods using cross-validation on four distinct datasets: (a) WBCD, (b) heart disease, (c) neural spike train, and (d) Pima Indians diabetes datasets.} This figure presents the average F1 score, ROC AUC, and Cohen's Kappa values obtained through 5- or 10-fold cross-validation for DNNs and the two DNN-based projection techniques: Hessian method and the proposed method. Both methods are referred to as DNN-based methods because they rely on the eigendirections derived from the Hessian matrix evaluated on a trained deep neural network. In subfigures (b) and (c), the DNN-based projection methods outperform the standard DNNs, while in subfigures (a) and (d), the standard DNNs show superior performance. This difference arises because DNNs utilize the entire feature space, while the projection methods operate within a reduced eigenspace, leveraging only one or two key eigendirections. When the excluded eigendirections in the full space contain additional relevant information, the DNNs outperform the projection methods; conversely, when the excluded features are primarily noise, the projection methods yield better performance by focusing on the most informative directions. }
  \label{fig:dnn_performances}
\end{figure*}

In contrast, datasets that do not fit the ideal conditions of isotropy and dominant leading eigenvalues may present challenges for the proposed method. For instance, if the within-class covariance matrix elements are random or lack a pattern approximating proportionality to the identity matrix, which can occur in unnormalized datasets or datasets with high redundancy among attributes, the method's effectiveness may be compromised. Additionally, if the covariance and/or Hessian matrices exhibit relatively uniform eigenvalues, where the leading eigenvalues are not sufficiently dominant, the method may not perform as effectively. 

In our experiments, the proposed method outperforms PCA and the Hessian method by comprehensively addressing both LDA criteria—maximizing between-class mean distance and minimizing within-class variances. Unlike PCA, which predominantly focuses on the former and lacks the guidance of class labels, our supervised approach considers both aspects. Despite the computational efficiency associated with unsupervised dimension reduction methods \citep{shen2014generalized}, our approach demonstrates the added value of incorporating class labels. This key insight also underlies our outperformance of KPCA, which, despite operating on non-linearities, remains essentially unsupervised in nature. While the Hessian method concentrates on minimizing within-class variances, our method optimally combines the strengths of PCA and the Hessian, effectively identifying feature space directions that enhance both between-class separation and within-class compactness.

While the proposed method employs the LDA criteria, it surpasses LDA itself in all cases. Figures \ref{fig:train_comparison} and \ref{fig:test_comparison} visually demonstrates the advantages of the proposed method over LDA. LDA is limited to a one-dimensional projection for binary classification problems \cite{ye2004two}, where it seeks to identify a single direction that simultaneously satisfies the two criteria. Conversely, the proposed method splits the task of meeting the criteria into two directions. The utilization of higher dimensionality in the proposed method increases the likelihood of discovering class separability, aligning with Cover's theorem \citep{cover1965geometrical}. KDA, operating on a non-linear mode, outperforms LDA in all cases, yet it remains fundamentally confined to one dimension, restricting its effectiveness in capturing intricate class-specific patterns compared to our proposed method.

Our method outperforms LOL by leveraging non-linear modes of operation, providing a distinct advantage in capturing complex patterns beyond the linear capabilities of LOL. UMAP exhibits good class separability with the widest margin as shown in Figures \ref{fig:train_comparison} and \ref{fig:test_comparison}. However, it is important to note that UMAP is not inherently a classification technique, and thus, it fails to generalize well to new data in Figure \ref{fig:test_comparison} . Similar limitations apply to LLE, which, although effective in revealing local data structures, lacks the inherent capability for classification and generalization.

Figures \ref{fig:train_comparison} and \ref{fig:test_comparison}  also underscores the simplicity and interpretability of linear SVMs as basic linear classifiers in dealing with low-dimensional data. The figure shows clear visualizations of the SVM's decision boundaries and separation achieved through the various dimensionality reduction methods. Notably, our proposed method combined with linear SVMs offers valuable insights into the decision-making process within the underlying DNN. It projects data into a space defined by the leading Hessian direction, which captures the key curvature directions in the DNN's loss landscape. This allows us to visualize which features have the greatest impact on the DNN’s decision-making process (Figure \ref{fig:param_contributions}, right subfigure).  SVM decision boundaries then illustrate class separation in this reduced space (Figures \ref{fig:train_comparison} and \ref{fig:test_comparison}, bottom right subfigures). This combined approach bridges the gap between complex decision boundaries of DNNs and comprehensible, highly accurate decision processes within a 2D space.

In another perspective, our proposed method highlights the effective use of DNNs for data projection and representation. In this approach, DNNs serve as a classifier whose decision-making process can be encapsulated in the Hessian matrix. By extracting the leading eigenvector from this matrix, we capture significant classification information, which is then used as one of the key projection directions in our proposed method. This process leverages the deep learning model's ability to distill complex patterns into a compact and informative representation. This aligns with the findings of previous work, such as that by Hinton and Salakhutdinov \cite{hinton2006reducing}, which highlights the potential of neural networks in enhancing data representation and separability.

The proposed method shows promising performance, but it has certain limitations. The applicability of our approach is based on the premise that a DNN model is already in place. Our method operates on top of the underlying DNN, so any shortcomings or biases in the DNN's performance will naturally reflect in the results obtained from our approach. The limitation of our method lies in its dependence on the quality and accuracy of the underlying DNN. 

Another limitation of the proposed method arises from its reliance on relatively ideal dataset conditions for optimal performance. In practical applications, not all datasets will meet these criteria. For instance, datasets with high redundancy among attributes, lacking proper normalization, or involving models that are not sufficiently sloppy may not exhibit the desired isotropy or dominance of leading eigenvalues. In such cases, the method's performance may degrade, potentially limiting its applicability to a broader range of datasets.

An important trajectory for future work involves investigating the extension of our methodology to accommodate various loss functions beyond binary cross-entropy. The mathematical derivation in our current work relies on the elegant relationship between binary cross-entropy loss and within-class variances. Even though the binary cross-entropy loss function is widely recognized as the standard loss function for binary classification models \citep{wang2024detection}, exploring the adaptability of our method to different loss functions will contribute to a more comprehensive understanding of the method's versatility, but requires careful scrutiny to establish analogous connections. 

Simultaneously, we recognize the need to extend our methodology from binary to multiclass classification. The binary classification focus in this work stems from foundational aspects guiding our formal proof, which is designed around binary assumptions to facilitate a streamlined and elegant derivation process. In particular, the use of binary cross-entropy as the loss function and the utilization of a linear SVM for evaluation inherently adhere to binary classification. Moving forward, careful exploration is needed to adapt our approach to multiclass scenarios to ensure its applicability and effectiveness across a broader range of classification tasks.

It is also important to validate our method on more complex and modern datasets. In this regard, another promising direction for future work is the application of our method to classification problems involving image datasets \citep{alaboodi2024lightweight, decoodt2024transfer, vaghefi2024exploration, madhusudhan2024detection, salowe2024utilizing, rainio2024comparison, naik2024herbid, kadam2024smart, yang2024dbformer}, audio datasets \citep{downward2023aeropsd, bannour2023optimizing, narayanan2024bioacoustic, cai2024unveiling, navine2024all, huddart2024solicited}, and their integration with additional features such as patient information or demographic data \citep{yadav2023comprehensive, kodipalli2024evaluation}. Leveraging datasets that involve more intricate patterns and higher-dimensional data will be instrumental in demonstrating the full potential and scalability of our methodology.

\section{Conclusion}

In this paper, we have made several contributions to the fields of binary classification and dimensionality reduction:

\begin{itemize}
    \item \textbf{Theoretical insight}: We provide a formal framework linking the eigenanalysis of covariance and Hessian matrices to LDA criteria, which enhances the understanding of class separability.
    \item \textbf{Novel method}: We propose a method that integrates covariance and Hessian matrices for dimensionality reduction, aimed at improving class separability in binary classification tasks.
    \item \textbf{Data conditions}: We define certain ideal dataset conditions under which our method performs optimally, ensuring that the theoretical assumptions are met and maximizing the method's effectiveness in practical applications. 
    \item \textbf{Empirical validation}: Our method, tested across various datasets, generally outperforms existing techniques, highlighting its potential robustness and effectiveness under the right conditions.
    \item \textbf{Practical utility}: The method retains a computational complexity comparable to traditional approaches, suggesting its practicality in real-world applications.
     \item \textbf{Improved explainability}: By integrating our method with linear SVMs, we enhance the explainability and interpretability of deep neural networks, addressing their typical opacity and facilitating better understanding of their decision-making processes.
\end{itemize}

\section{Reproducibility statement}

The detailed proofs for the theoretical foundations, emphasizing the maximization of between-class mean distance and minimization of within-class variance, are provided in \ref{full-proof}. The complete source code and datasets for our experiments are accessible through the following links:

\begin{enumerate}

\item WBCD dataset experiment: \href{https://colab.research.google.com/drive/19Wny8Mvb40mK8KEt33uHjM9HQt-IZYod?usp=sharing}{Link to Colab notebook}

\item Heart disease dataset experiment: \href{https://colab.research.google.com/drive/1TCo5L7W10OsWNL8oLpjQTBNv4hkft_62?usp=sharing}{Link to Colab notebook} 

\item Neural spike train dataset experiment: \href{https://colab.research.google.com/drive/1QFR0KbzteLo-XXAt12xYB3kL3FSv6u6_?usp=sharing}{Link to Colab notebook}

\item Pima Indians diabetes dataset experiment: \href{https://colab.research.google.com/drive/1opbwsNihkZRIcaM1AVukR5IqdG41ijmC?usp=sharing}{Link to Colab notebook}

\item Illustrative case study on WBCD dataset: \href{https://colab.research.google.com/drive/1-8MKDuXhqT0HcYalNPo1PP4AG3PH2AlB?usp=sharing}{Link to Colab notebook}

\end{enumerate}

These notebooks contain the complete source code along with the datasets accessed from the same Google Drive account, facilitating easy reproduction and comprehension of our results.

\section{Acknowledgments}

This research is supported by the European Union’s Horizon 2020 research and innovation programme under grant agreement Sano No 857533 and the project of the Minister of Science and Higher Education "Support for the activity of Centers of Excellence established in Poland under Horizon 2020" on the basis of the contract number MEiN/2023/DIR/3796.

\bibliographystyle{elsarticle-num} 
\bibliography{references}






\appendix

\newpage
\onecolumn
\section{Full proofs: maximizing the squared between-class mean distance and minimizing the within-class variance}
\label{full-proof}

\subsection{Proof of Theorem \ref{thm:Theorem1} - Maximizing covariance for maximizing squared between-class mean distance}

To prove that maximizing the variance will maximize the squared between-class mean distance, we start by considering two sets of 1D data points representing two classes, denoted as \( C_1 \) and \( C_2 \), each consisting of \( n \) samples. The data in \( C_1 \) and \( C_2 \) are centered around their respective means, \( \mu_1 \) and \( \mu_2 \). Let \( \mu \) denote the overall mean of the combined data, expressed as:
\[
\mu = \frac{1}{2n} \left( \sum_{i=1}^{n} x_i + \sum_{i=n+1}^{2n} x_i \right).
\]

The between-class mean distance, denoted as $d$, represents the separation between the means of $C_1$ and $C_2$. We can express the means as $\mu_1 = \mu - \frac{d}{2}$ and $\mu_2 = \mu + \frac{d}{2}$, where $\mu$ is effectively located in the middle, equidistant from both $\mu_1$ and $\mu_2$.

Furthermore, the variances of $C_1$ and $C_2$, are equal, denoted as $\sigma_1^2 = \sigma_2^2 = \sigma_w^2$, and the combined data from $C_1$ and $C_2$ has a variance of $\sigma^2$. The variance of the combined data is given by:
\[
\sigma^2 = \frac{1}{2n-1}\left(\sum_{i=1}^{n}(x_i-\mu)^2+\sum_{i=n+1}^{2n}(x_i-\mu)^2\right)
\]

To simplify the expression, we substitute $\mu = \mu_1 + \frac{1}{2}d$ and $\mu = \mu_2 - \frac{1}{2}d$. This yields:

\begin{align*}
s^2 &= \frac{1}{2n-1}\left(\sum_{i=1}^{n}\left(x_i-\left(\mu_1 + \frac{1}{2}d\right)\right)^2+\sum_{i=n+1}^{2n}\left(x_i-\left(\mu_2 - \frac{1}{2}d\right) \right)^2\right) \\
&= \frac{1}{2n-1}\left(\sum_{i=1}^{n}\left(\left(x_i-\mu_1 \right) - \frac{1}{2}d\right)^2+\sum_{i=n+1}^{2n}\left(\left(x_i-\mu_2 \right) + \frac{1}{2}d \right)^2\right) \\
\begin{split}
&= \frac{1}{2n-1}\left(\sum_{i=1}^{n}\left(\left(x_i-\mu_1 \right)^2 - \left(x_i-\mu_1 \right)d + \frac{1}{4}d^2\right)\right) \\
&+ \left. \sum_{i=n+1}^{2n}\left(\left(x_i-\mu_2 \right)^2 + \left(x_i-\mu_2 \right)d + \frac{1}{4}d^2 \right)\right)  \\
&= \frac{1}{2n-1}\left(\sum_{i=1}^{n}\left(\left(x_i-\mu_1 \right)^2  + \frac{1}{4}d^2\right) - \sum_{i=1}^{n} \left( \left(x_i-\mu_1 \right)d \right) \right)\\
&+ \left. \sum_{i=n+1}^{2n}\left(\left(x_i-\mu_2 \right)^2  + \frac{1}{4}d^2 \right)\right) + \sum_{i=n+1}^{2n} \left(\left(x_i-\mu_2 \right)d \right) \\
&= \frac{1}{2n-1}\left(\sum_{i=1}^{n}\left(\left(x_i-\mu_1 \right)^2  + \frac{1}{4}d^2\right) + \left. \sum_{i=n+1}^{2n}\left(\left(x_i-\mu_2 \right)^2  + \frac{1}{4}d^2 \right)\right)\right) \\
&+ \sum_{i=n+1}^{2n} \left(\left(x_i-\mu_2 \right)d \right) - \sum_{i=1}^{n} \left( \left(x_i-\mu_1 \right)d \right) \\
&= \frac{1}{2n-1}\left(\sum_{i=1}^{n}\left(\left(x_i-\mu_1 \right)^2  + \frac{1}{4}d^2\right) + \left. \sum_{i=n+1}^{2n}\left(\left(x_i-\mu_2 \right)^2  + \frac{1}{4}d^2 \right)\right) \right) \\
&+ \left( \sum_{i=n+1}^{2n} \left(x_i-\mu_2 \right) - \sum_{i=1}^{n}  \left(x_i-\mu_1 \right) \right)d\\
\end{split}
\end{align*}

Now, let’s consider the properties of the two subsets \( C_1 \) and \( C_2 \). By the zero-sum property, we have:
\[
\sum_{i=1}^{n} (x_i - \mu_1) = 0 \quad \text{and} \quad \sum_{i=n+1}^{2n} (x_i - \mu_2) = 0.
\]
Substitute these properties into the variance equation:

\begin{align*}
\sigma^2 &= \frac{1}{2n-1}\left(\sum_{i=1}^{n}\left(\left(x_i-\mu_1 \right)^2 + \frac{1}{4}d^2\right) + \sum_{i=n+1}^{2n}\left(\left(x_i-\mu_2 \right)^2 + \frac{1}{4}d^2 \right)\right) \\
&= \frac{1}{2n-1}\left(\sum_{i=1}^{n}\left(x_i-\mu_1 \right)^2 + \sum_{i=n+1}^{2n}\left(x_i-\mu_2 \right)^2 + \frac{1}{2}n\cdot d^2 \right) \\
&\approx \frac{1}{2 \left( n-1 \right)}\left(\sum_{i=1}^{n}\left(x_i-\mu_1 \right)^2 + \sum_{i=n+1}^{2n}\left(x_i-\mu_2 \right)^2  \right) + \frac{1}{4}d^2 \\
&= \frac{1}{2}\left(\frac{\sum_{i=1}^{n}\left(x_i-\mu_1 \right)^2}{\left( n-1 \right)} + \frac{\sum_{i=n+1}^{2n}\left(x_i-\mu_2 \right)^2}{\left( n-1 \right)}  \right) + \frac{1}{4}d^2 \\
&= \frac{1}{2}\left(\sigma_1^2 + \sigma_2^2  \right) + \frac{1}{4}d^2.
\end{align*}

Now, considering the data points representing the projected data onto an (Eigen)vector, we can utilize \hyperref[VRPT]{the variance ratio preservation property for projection onto a vector}, which establishes the relationship between $\sigma^2$, $\sigma_1^2$, and $\sigma_2^2$ as follows:
\begin{align*}
\lambda_1 = \frac{\sigma_1^2}{\sigma^2}, \quad
\lambda_2 = \frac{\sigma_2^2}{\sigma^2},
\end{align*}
or equivalently,
\begin{align*}
\sigma_1^2 = \lambda_1 \cdot \sigma^2, \quad
\sigma_2^2 = \lambda_2 \cdot \sigma^2,
\end{align*}
where each of $\lambda_1$ and $\lambda_2$ is a constant between 0 and 1, determined by the distribution of the original data being projected.

Substituting this equation into the previous expression, we have:

\begin{align*}
\sigma^2 &= \frac{1}{2} \left( \lambda_1 + \lambda_2  \right) \sigma^2 + \frac{1}{4}d^2.
\end{align*}

Let \( \lambda = \frac{1}{2} \left( \lambda_1 + \lambda_2  \right) \). Then:

\begin{align*}
\sigma^2 &= \lambda \cdot \sigma^2 + \frac{1}{4}d^2.
\end{align*}

Let's rearrange the equation by moving $2\lambda \cdot \sigma^2$ to the left side:

\begin{align*}
\sigma^2 - \lambda \cdot \sigma^2 &= \frac{1}{4}d^2.
\end{align*}

Combining like terms:

\begin{align*}
(1 - \lambda) \cdot \sigma^2 &= \frac{1}{4}d^2.
\end{align*}

To solve for $\sigma^2$, divide both sides by $(1 - \lambda)$:

\begin{align*}
\sigma^2 &= \frac{\frac{1}{4}d^2}{1 - \lambda} = \frac{r^2}{1 - \lambda}
\end{align*}

where $r = \frac{1}{2}d = \mu - \mu_1 = \mu_2 - \mu$.

We observe that the sign of $\sigma^2$ and $d^2$ will be the same since the denominator $1 - \lambda$ is always positive (as $0 < \lambda < 1$). Therefore, $\sigma^2$ is linearly proportional to $d^2$.

Hence, maximizing the variance ($\sigma^2$) will maximize the squared between-class mean distance ($d^2$) as desired.

\subsection{Proof of Theorem \ref{thm:Theorem2} - Maximizing hessian for minimizing within-class variance}
 
We aim to prove that maximizing the Hessian will minimize the within-class variance. Let $\theta$ denote a parameter of the classifier.

We define the within-class variance as the variance of a posterior distribution $p_\theta(\theta \mid c_i)$, which represents the distribution of the parameter $\theta$ given a class $c_i$. We denote the variance of this posterior distribution as $\sigma_{post}^2$.

Recall that our Hessian is given by:

\begin{align*}
\mathrm{H}_\theta &= - \left[\nabla_\theta^2 \log p_\theta(c_i \mid \theta) \right]
\end{align*}

However, according to \cite{barshan2020relatif}, we can approximate the Hessian using Fisher information:

\begin{align*}
\mathrm{H}_\theta &\approx \mathbb{E}_{p_\theta} \left[\left(\nabla_\theta \log p_\theta(c_i \mid \theta) \right)^2 \right]
\end{align*}

Assuming that a known normal distribution underlies the likelihood $p_\theta(c_i \mid \theta)$, i.e.,

\begin{equation*}
p_\theta(c_i \mid \theta) = \frac{1}{\sqrt{2\pi\sigma^2}} \exp\left(-\frac{(\theta-\mu)^2}{2\sigma^2}\right)
\end{equation*}

where $\mu$ and $\sigma$ are known constants, we can compute the Hessian as follows:

\begin{align*}
\mathrm{H}_\theta &= \mathbb{E}_{p_\theta} \left[\left(\nabla_\theta \log \frac{1}{\sqrt{2\pi\sigma^2}} \exp\left(-\frac{(\theta-\mu)^2}{2\sigma^2}\right) \right)^2 \right] \\
&= \mathbb{E}_{p_\theta} \left[\left(\frac{\theta - \mu}{\sigma^2}\right)^2 \right] \\
&= \frac{1}{\sigma^4} \mathbb{E}_{p_\theta} \left[(\theta - \mu)^2 \right] \\
&= \frac{1}{\sigma^4} \sigma^2 \\
&= \frac{1}{\sigma^2}
\end{align*}

Now, let us assume that the evidence $p(c_i)$ is a known constant $\rho$. We also assume that the prior distribution $p(\theta)$ follows a uniform distribution within the plausible range of $\theta$, which is bounded by a known minimum value $\theta_{\min}$ and maximum value $\theta_{\max}$. Formally:

\begin{equation*}
p_\theta(\theta) =
\begin{cases}
\dfrac{1}{\theta_{\max} - \theta_{\min}}, & \text{if}\ \theta_{\min} \leq \theta \leq \theta_{\max} \\
0, & \text{otherwise}
\end{cases}
\end{equation*}

Based on Bayes' formula and the given assumptions, the posterior distribution within the plausible range of $\theta$ is:

\begin{align*}
p_\theta(\theta \mid c_i) &= \frac{p_\theta(c_i \mid \theta) \cdot p_\theta(\theta)}{p(c_i)} \\
&= \frac{1}{\rho (\theta_{\max} - \theta_{\min})} \cdot \frac{1}{\sqrt{2\pi\sigma^2}} \exp\left(-\frac{(\theta-\mu)^2}{2\sigma^2}\right)
\end{align*}

This implies that the posterior distribution $p_\theta(\theta \mid c_i)$ is a normal distribution with mean $\mu$ and variance $\sigma^2$:

\begin{equation*}
p_\theta(\theta \mid c_i) \sim \mathcal{N}(\mu,\,\sigma^{2})
\end{equation*}

Therefore, the within-class variance $\sigma_{\text{post}}^2$ of the posterior distribution is equal to $\sigma^2$. 

Combining the previous result with the Hessian calculation, we can conclude that:

\begin{align*}
\mathrm{H}_\theta &= \frac{1}{\sigma^2} = \frac{1}{\sigma_{\text{post}}^2}
\end{align*}

Hence, maximizing the Hessian ($\mathrm{H}_\theta$) will minimize the within-class variances ($\sigma_{\text{post}}^2$) as desired.

\subsection{Variance ratio preservation property for projection onto a vector}
\label{VRPT}

The property is formalized in Theorem \ref{thm:VRPT}, which states that the ratio of variances between any two subsets of 1D data is preserved when the data is projected onto a 2-dimensional vector.

\begin{theorem}[Variance ratio preservation upon projection onto a vector]
\label{thm:VRPT}
Let \( X = \{x_1, x_2, \ldots, x_n\} \) be a set of 1D data points, with two subsets \( X_1 = \{x_1, \ldots, x_m\} \) and \( X_2 = \{x_{m+1}, \ldots, x_n\} \), having variances \(\sigma_{X_1}^2\) and \(\sigma_{X_2}^2\), respectively. Consider projecting \( X \) onto a unit vector \( \mathbf{v} = (v_1, v_2) \) in a 2-dimensional space, resulting in a set of projected points \( Y = \{y_1, y_2, \ldots, y_n\} \), with corresponding projected subsets \( Y_1 \) and \( Y_2 \).

Then the ratio of variances between the two projected subsets \( Y_1 \) and \( Y_2 \) is equal to the ratio of the variances between their original subsets \( X_1 \) and \( X_2 \):

\[
\frac{\sigma_{Y_2}^2}{\sigma_{Y_1}^2} = \frac{\sigma_{X_2}^2}{\sigma_{X_1}^2}.
\]
\end{theorem}

\begin{proof}
For a point \((x_i, 0)\) on the first axis, its projection onto \( \mathbf{v} = (v_1, v_2) \) is given by \( y_i = x_i \cdot v_1 \). Consequently, the distance between two projected points \( y_i \) and \( y_j \) is \( |y_i - y_j| = |x_i \cdot v_1 - x_j \cdot v_1| = |v_1| \cdot |x_i - x_j| \). Thus, the distances between pairs of points in the original space and their projections are scaled uniformly by \( |v_1| \).

Next, consider the variance of the projected subset \( Y_1 \), denoted as \( \sigma_{Y_1}^2 \). We can compute it using the formula \( \sigma_{Y_1}^2 = \frac{1}{m-1} \sum_{i=1}^{m} (y_i - \mu_{Y_1})^2 \), where \( \mu_{Y_1} \) is the mean of the subset \( Y_1 \). Substituting \( y_i = x_i \cdot v_1 \), and using the fact that \( (y_i - \mu_{Y_1}) = v_1 \cdot (x_i - \mu_{X_1}) \), we obtain \( \sigma_{Y_1}^2 = v_1^2 \cdot \frac{1}{m-1} \sum_{i=1}^{m} (x_i - \mu_{X_1})^2 = v_1^2 \cdot \sigma_{X_1}^2 \).

Similarly, the variance for the second subset \( Y_2 \) is given by \( \sigma_{Y_2}^2 = v_1^2 \cdot \sigma_{X_2}^2 \).

Taking the ratio of the variances of the projected subsets \( Y_1 \) and \( Y_2 \):

\[
\frac{\sigma_{Y_2}^2}{\sigma_{Y_1}^2} = \frac{v_1^2 \cdot \sigma_{X_2}^2}{v_1^2 \cdot \sigma_{X_1}^2} = \frac{\sigma_{X_2}^2}{\sigma_{X_1}^2}.
\]

This establishes that the ratio of variances is preserved in the projection, as required.
\end{proof}

\newpage

\section{Multidimensional extension of Theorem 1 for maximizing squared between-class mean distance via eigenanalysis of the covariance matrix}
\label{extension}

To extend Theorem \ref{thm:Theorem1} to multidimensional data, we consider a dataset where each sample is represented by a \(D\)-dimensional column vector \(\vec{\theta}=\left[\theta_{1}, \theta_{2}, \cdots, \theta_{D}\right]^{T} \).
The dataset of $2n$ samples is divided into two subsets: \(C_{1} = \{\vec{x}_{i} \in \mathbf{R}^{D} \mid i=1,2,\ldots,n\}\) and \(C_{2} = \{\vec{y}_{j} \in \mathbf{R}^{D} \mid j=1, 2,\ldots, n\}\), each containing \(n\) samples. The data in \(C_1\) and \(C_2\) follow the same underlying distribution, centered around their respective means, \(\vec{\mu}_1\) and \(\vec{\mu}_2\), resulting in the combined data from \(C_1\) and \(C_2\) being centered around an overall mean of \(\vec{\mu}\). The covariances for \(C_1\) and \(C_2\) are denoted as \(S_1\) and \(S_2\), respectively, and the combined data has an overall covariance of \(S\). The individual distributions for the vectors $\vec{X}$ and $\vec{Y}$ are isotropic about their respective means, meaning their covariances are proportional to the unit matrix in \(D\) dimensions, i.e., \( S_1 = \sigma_1^{2} I_{D} \) and \( S_2 = \sigma_2^{2} I_{D} \). The distance between the class means, \( d = \left\|\left(\vec{\mu}_{1}-\vec{\mu}_{2}\right)\right\| \), indicates the separation between \(\vec{\mu}_1\) and \(\vec{\mu}_2\).

Define the sample means:

\begin{equation}
\vec{\mu}=\frac{1}{2n} \sum_{k=1}^{2n} \vec{\theta}_{k}=\frac{n}{2n} \vec{\mu}_{1}+\frac{n}{2n} \vec{\mu}_{2} =\frac{1}{2} \left( \vec{\mu}_{1}+\vec{\mu}_{2} \right)
\label{eq:B1}
\end{equation}
where

\begin{equation}
\vec{\mu}_{1}=\frac{1}{n} \sum_{i=1}^{n} \vec{x}_{i} \text { and } \vec{\mu}_{2}=\frac{1}{n} \sum_{j=1}^{n} \vec{y}_{j}
\label{eq:B2}
\end{equation}
The biased sample covariance matrix of the combined training set is then given by:

\begin{equation}
S=\frac{1}{2n} \sum_{k=1}^{2n}\left(\vec{\theta}_{k}-\vec{\mu}\right)\left(\vec{\theta}_{k}-\vec{\mu}\right)^{T}=\frac{1}{2n} \sum_{k=1}^{2n} \vec{\theta}_{k} \vec{\theta}_{k}^{T}-\vec{\mu} \vec{\mu}^{T}
\label{eq:B3}
\end{equation}
Expanding this yields:

\begin{equation}
S=\frac{1}{2n}\left[\sum_{i=1}^{n} \vec{x}_{i} \vec{x}_{i}^{T}+\sum_{j=1}^{n} \vec{y}_{j} \vec{y}_{j}^{T}\right]-\vec{\mu} \vec{\mu}^{T}
\label{eq:B4}
\end{equation}
Now:

\begin{equation}
\sum_{i=1}^{n} \vec{x}_{i} \vec{x}_{i}^{T}=\sum_{i=1}^{n} \vec{x}_{i} \vec{x}_{i}^{T}-\left(n \right) \vec{\mu}_{1} \vec{\mu}_{1}^{T}+\left(n \right) \vec{\mu}_{1} \vec{\mu}_{1}^{T}= n \left(S_{1}+\vec{\mu}_{1} \vec{\mu}_{1}^{T}\right)
\label{eq:B5}
\end{equation}
and similarly:

\begin{equation}
\sum_{j=1}^{n} \vec{y}_{j} \vec{y}_{j}^{T}=n \left(S_{2}+\vec{\mu}_{2} \vec{\mu}_{2}^{T}\right)
\label{eq:B6}
\end{equation}
where \(S_{1}\) and \(S_{2}\) are the (biased) sample covariances of the class 1 and class 2 subsets respectively. Substituting \eqref{eq:B1}, \eqref{eq:B5}, and \eqref{eq:B6} in \eqref{eq:B4}, after some algebra, leads to:

\begin{equation}
S=\frac{n}{2n} S_{1}+\frac{n}{2n} S_{2}+\frac{n^2}{\left(2n\right)^{2}}\left(\vec{\mu}_{1}-\vec{\mu}_{2}\right)\left(\vec{\mu}_{1}-\vec{\mu}_{2}\right)^{T} =  \frac{1}{2} S_{1}+\frac{1}{2} S_{2} +\frac{1}{4}\left(\vec{\mu}_{1}-\vec{\mu}_{2}\right)\left(\vec{\mu}_{1}-\vec{\mu}_{2}\right)^{T}
\label{eq:B7}
\end{equation}

Theorem \ref{thm:Theorem1} is supposed to hold for the projection of the data onto an arbitrary 1D subspace of the full D-dimensional data space, and to imply that maximizing \(\sigma^{2}\) by choosing the direction of this subspace appropriately will also be the direction which maximizes the difference between the projected sample means. Fortunately, since \( S_1 = \sigma_1^{2} I_{D} \) and \( S_2 = \sigma_2^{2} I_{D} \), we have:

\begin{equation}
S = \frac{1}{2} \sigma_1^{2} I_{D} + \frac{1}{2} \sigma_2^{2} I_{D} +\frac{1}{4}\left(\vec{\mu}_{1}-\vec{\mu}_{2}\right)\left(\vec{\mu}_{1}-\vec{\mu}_{2}\right)^{T}
\label{eq:B10}
\end{equation}
Multiplying both sides with \(\left(\vec{\mu}_{1}-\vec{\mu}_{2}\right)\) yields:

\begin{equation}
S \left(\vec{\mu}_{1}-\vec{\mu}_{2}\right) = \left(\frac{1}{2} \sigma_1^{2} + \frac{1}{2} \sigma_2^{2} +\frac{1}{4} d^{2}\right)\left(\vec{\mu}_{1}-\vec{\mu}_{2}\right)
\label{eq:B11}
\end{equation}
Please note that equation \eqref{eq:B11} forms an eigenequation where $S$ is the matrix of interest, the term $\left(\frac{1}{2} \sigma_1^{2} + \frac{1}{2} \sigma_2^{2} +\frac{1}{4} d^{2}\right)$ represents the eigenvalue, and $\left(\vec{\mu}_{1}-\vec{\mu}_{2}\right)$ is the eigenvector. This implies that the difference between the mean vectors $\left(\vec{\mu}_{1}-\vec{\mu}_{2}\right)$ is indeed an eigenvector of the full set's covariance matrix $S$. This alignment confirms that the projection maximizing the variance in the D-dimensional space will also maximize the between-class mean distance, ensuring the validity of Theorem \ref{thm:Theorem1} in higher dimensions. 

\newpage

\section{Additional dataset results}
\label{more-results}

\begin{figure}[!htb]
  \centering
  \includegraphics[width=14cm]{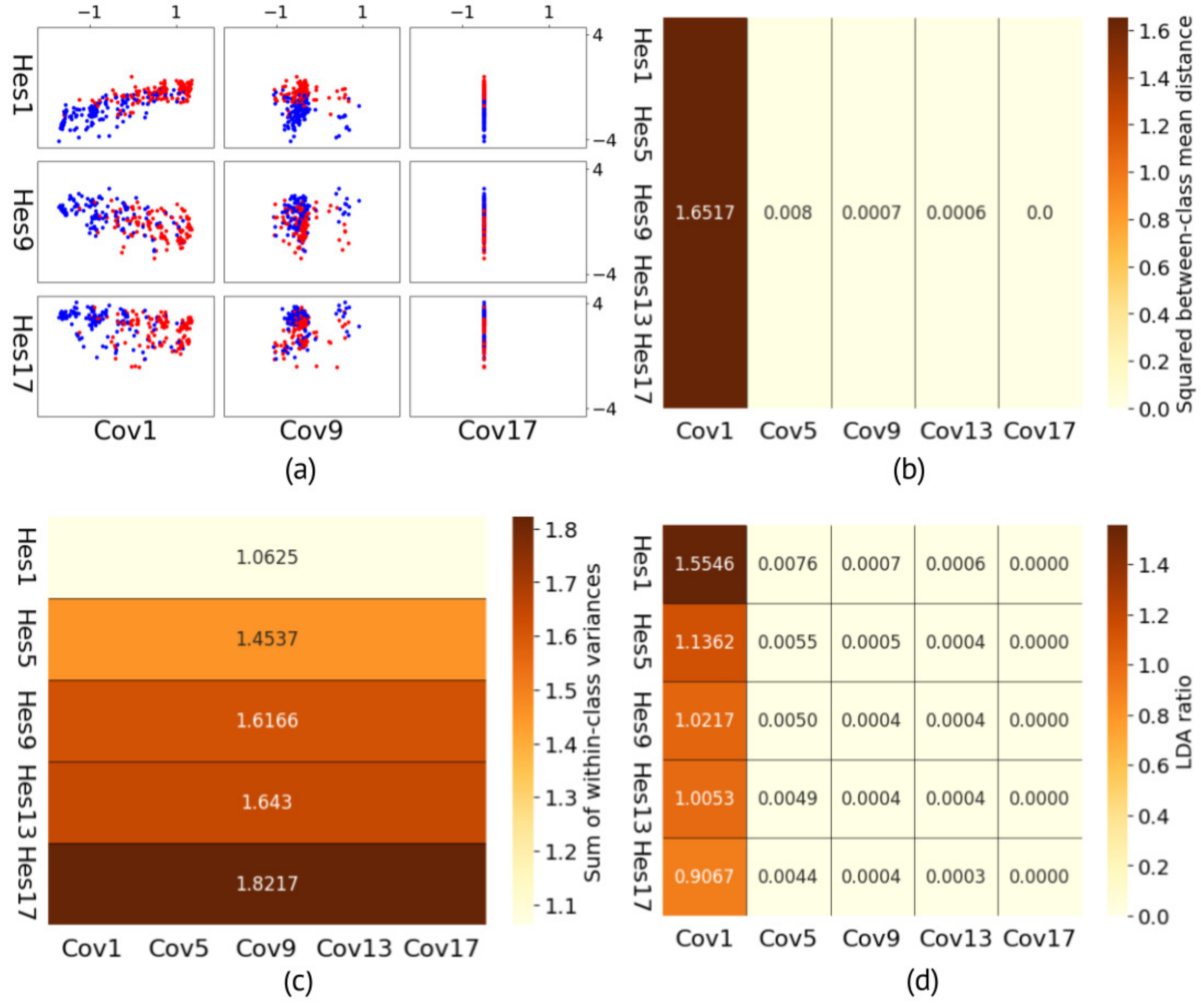}
  \caption{\textbf{Projection of the heart disease data into different combined spaces of the covariance and Hessian eigenvectors.}}
  \label{figx}
\end{figure}

\begin{figure}
  \centering
  \includegraphics[width=14cm]{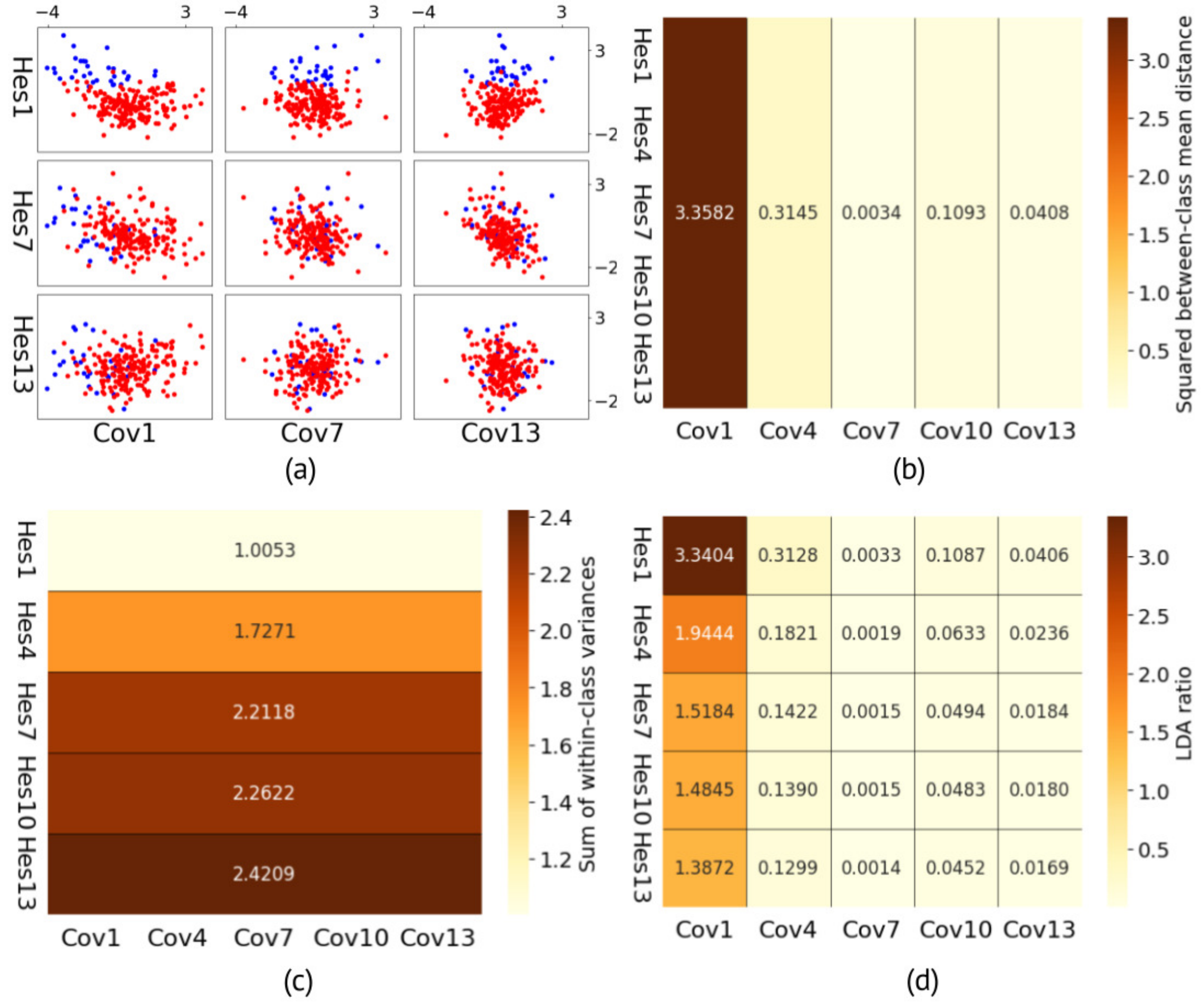}
  \caption{\textbf{Projection of the neural spike train data into different combined spaces of the covariance and Hessian eigenvectors.}}
  \label{figx}
\end{figure}

\begin{figure}
  \centering
  \includegraphics[width=14cm]{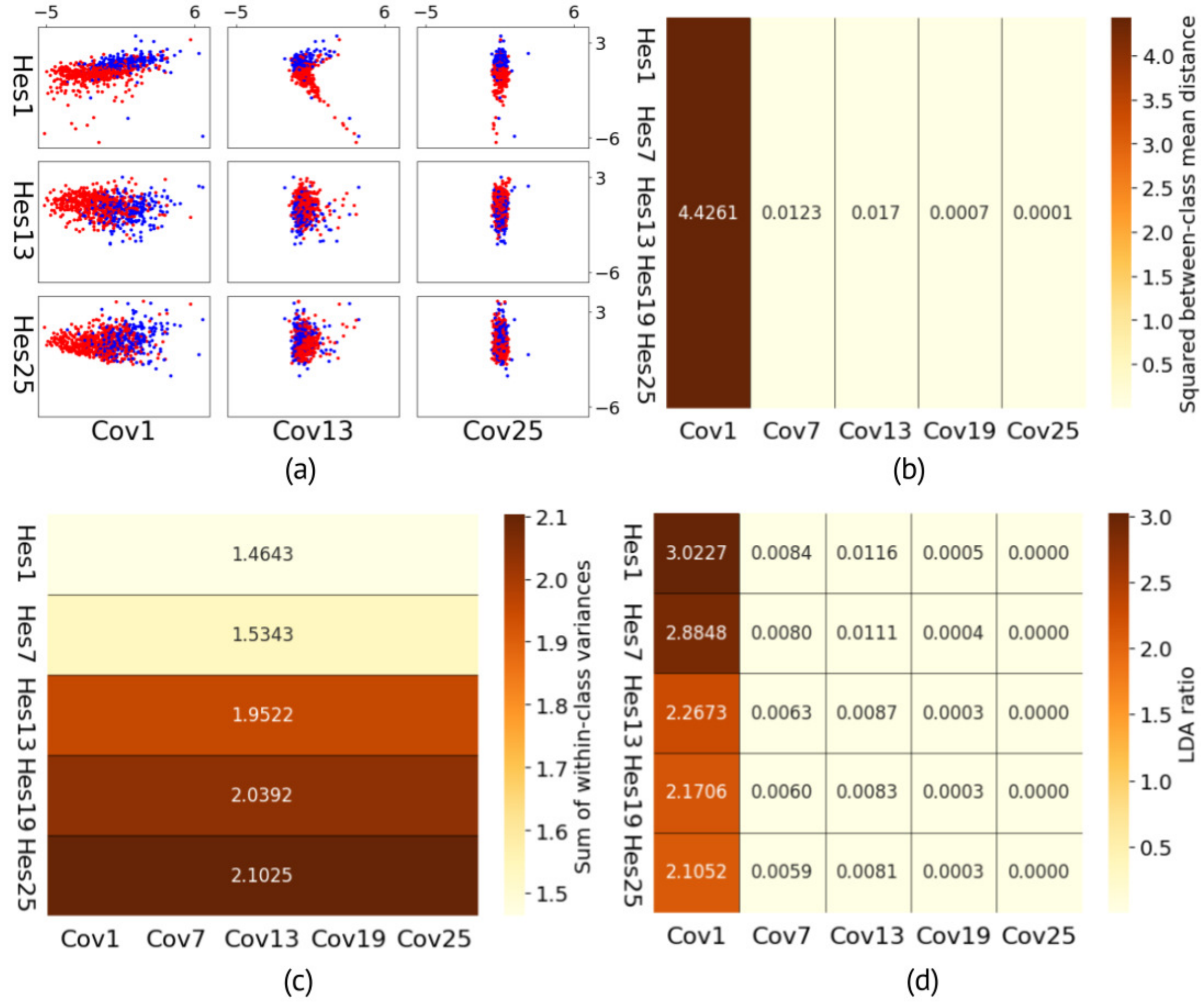}
  \caption{\textbf{Projection of the Pima Indians diabetes data into different combined spaces of the covariance and Hessian eigenvectors.} \textbf{(a)} Nine selected projection plots, each representing data projected onto a distinct space created by combining the first three covariance and first three Hessian eigenvectors. \textbf{(b)} Heatmap showing the squared between-class mean distance for projections onto varying combinations of covariance and Hessian eigenvectors. \textbf{(c)} Heatmap showing the sum of within-class variances for projections onto different combinations of covariance and Hessian eigenvectors. \textbf{(d)} Heatmap displaying the LDA ratio, representing the ratio between the squared between-class mean distances presented in (b) and the corresponding within-class variances shown in (c).}
  \label{figx}
\end{figure}

\newpage

\section{Within-class covariance matrices for unnormalized WBCD dataset}
\label{unnormalized_cov}

\begin{figure*}[h]
  \centering
  \includegraphics[width=16cm]{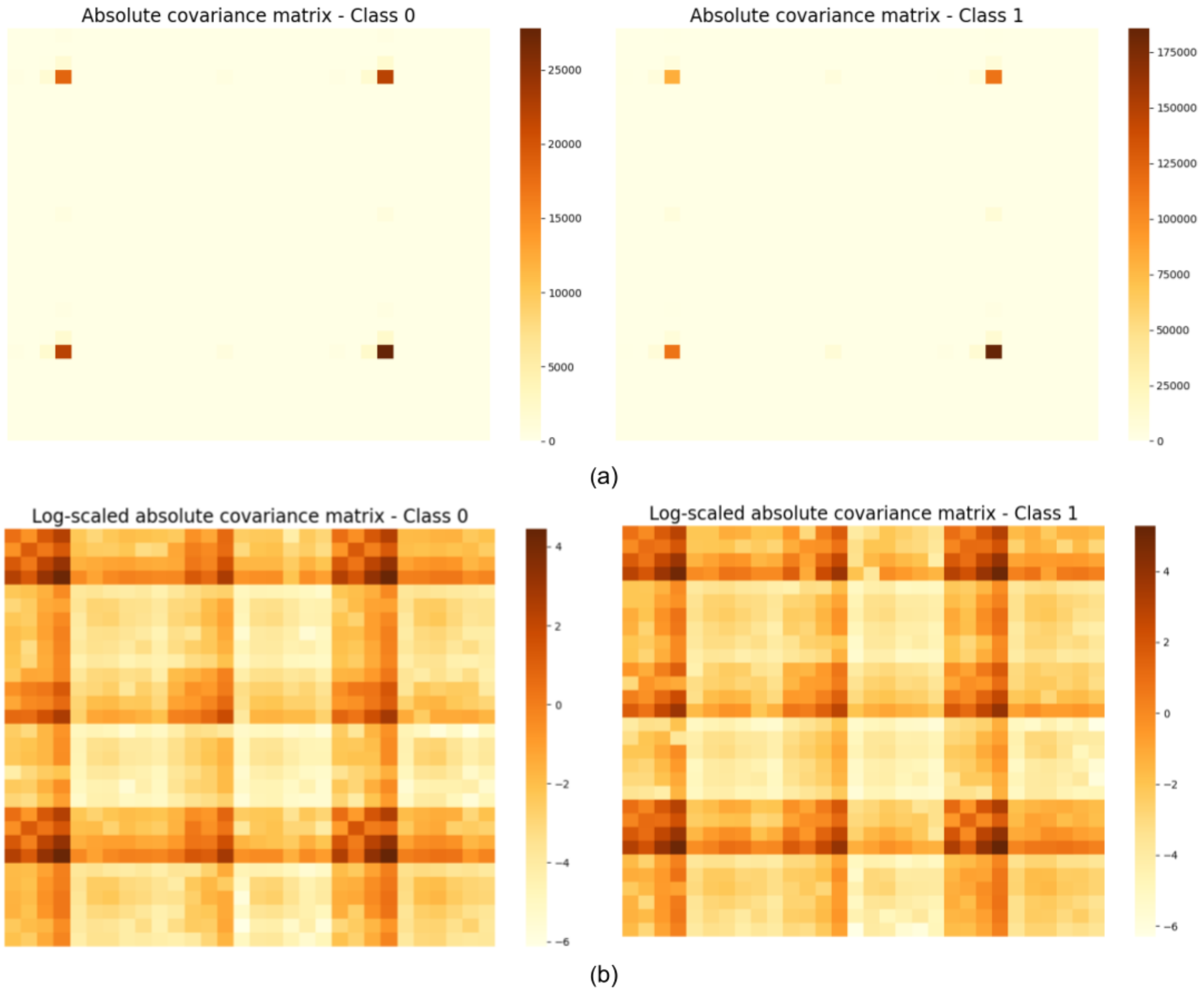}
  \caption{\textbf{Absolute covariance matrices for unnormalized WBCD dataset by class.} Panel (a) presents the absolute within-class covariance matrices for unnormalized WBCD dataset in decimal scale. It reveals that the variances are dominated by a couple of components, each on the order of \(10^5\), along with their significant covariance. These notable variances and covariance overshadow those of the other components, some of which are on the order of \(10^{-5}\). To enhance the visualization of variances and covariance structures, panel (b) utilizes a log scale. However, even with this adjustment, the covariance matrices reveal clusters without displaying a consistent pattern of uniform diagonal elements or isomorphism, clearly demonstrating that the lack of normalization has contributed to this irregular structure.}
  \label{fig:abs_cov_matrices_unnormalized}
\end{figure*}

\end{document}